\pdfoutput=1
\documentclass{article}
\usepackage[preprint]{neurips_2026}

\usepackage[utf8]{inputenc}
\usepackage[T1]{fontenc}
\usepackage{hyperref}
\usepackage{url}
\usepackage{booktabs}
\usepackage{amsfonts}
\usepackage{amsmath}
\usepackage{amssymb}
\usepackage{amsthm}
\usepackage{nicefrac}
\usepackage{microtype}
\usepackage{xcolor}
\usepackage{graphicx}
\usepackage{algorithm}
\usepackage{algpseudocode}
\usepackage{enumitem}
\usepackage{multirow}

\newtheorem{lemma}{Lemma}
\newtheorem{remark}{Remark}

\newcommand{\R}{\mathbb{R}}
\newcommand{\E}{\mathbb{E}}
\newcommand{\Var}{\mathrm{Var}}
\newcommand{\softmax}{\mathrm{softmax}}
\newcommand{\softplus}{\mathrm{softplus}}
\newcommand{\lamb}{\lambda}
\newcommand{\lobs}{\lambda^{\mathrm{obs}}}
\newcommand{\lhat}{\hat{\lambda}}
\newcommand{\lpost}{\lambda'}
\newcommand{\lpred}{\lambda''}

\title{Precision Tracked Transformer via Kalman Filtering, Kriging and Process Noise}

\author{%
  Bo Long, Deepak Agarwal, Jelena Markovic-Voronov, Yi Wang, Liuqing Li \\
  LinkedIn Core AI
}

\begin{document}

\maketitle

\begin{abstract}
The Transformer is the foundational building block of modern AI, yet offers no principled handling of \emph{uncertainty}, which is prevalent in real applications: cold-start tokens with sparse histories in sequential recommendation, heterogeneous signal quality in language models, and attention sinks induced by unconstrained softmax. Every token is treated with uniform confidence. We show this uniformity is a degenerate case of our \emph{Bayesian Filtering Transformer} (BFT): attention becomes precision-weighted kriging, the residual connection becomes a Kalman update with adaptive gain, and the FFN becomes a dynamics model propagating precision via a Jacobian--plus--process-noise rule. Observation precision comes from a parameter-free Restricted Maximum Likelihood (REML) estimator with a conjugate Bayesian prior. BFT replaces any Transformer layer with negligible overhead. On sequential recommendation, BFT applied to three major architectures yields significant gains on six benchmarks, with the largest improvements on cold-start users and rare items where uncertainty is highest. On supervised fine-tuning of large language models with noisy data, BFT improves robustness in two regimes: noisy supervision (token-label corruption in question answering) and noisy context (retrieval-augmented QA with real RAG distractors). A single principled modification---restoring precision---unlocks substantial headroom across both classical sequence-modeling and modern LLM regimes.
\end{abstract}

\section{Introduction}
\label{sec:intro}

The Transformer~\citep{vaswani2017attention} is the dominant architecture in modern AI: language models, recommenders, vision encoders, multimodal systems. Its self-attention mechanism computes a weighted combination of values mediated by query--key similarities. Although this simple design has proven remarkably universal, it makes an implicit assumption that is limiting: \emph{every token is equally reliable}.

Real applications routinely violate this assumption.
In \textbf{generative recommendation}~\citep{zhai2024hstu}, user sequences contain a mixture of high-signal interactions (careful purchases) and noisy clicks, while items range from popular products with rich histories to cold-start items observed only a handful of times~\citep{chen2022denoising,wang2021denoising}. The current Transformer does not have a principled, model-based way to down-weight a noisy item or to express uncertainty about a cold item. Heuristics applied at the input-construction stage are widely employed but hard to tune. This is one of the main reasons why Transformer-based sequential recommenders have struggled to deliver strong performance in cold-start scenarios.

In \textbf{LLM fine-tuning on real-world data}, domain corpora contain tokens of uneven quality; a standard loss weights every token identically, giving low-signal tokens the same gradient pull as high-signal ones~\citep{li2020dividemix}. The same uniform-confidence pathology shows up elsewhere in modern LLMs: \emph{attention sinks}~\citep{xiao2024efficient,sun2024massive,qiu2025gated} concentrate attention mass on uninformative early tokens, and \emph{position uncertainty} grows with retrieval depth in long-context RAG~\citep{liu2024lost}---in both cases, the architecture has no native channel through which to say ``this position is unreliable.'' Bespoke fixes have been proposed for each of these phenomena---denoising masks, attention calibration, sigmoid gates, sink tokens---but they remain isolated patches. We show that a single Bayesian formulation that propagates uncertainty through estimates provides a holistic solution and addresses these failure modes together.

\paragraph{Contributions.}
\begin{enumerate}[leftmargin=*,itemsep=1pt,topsep=1pt]
  \item \textbf{A unifying framework (Section~\ref{sec:framework}).} We formulate the Transformer as a three-step Bayesian filter---\emph{observe, update, predict}---and identify the precision $\lamb{=}1$ degeneracy as the common structural cause of attention sinks, noisy-feedback corruption, and cold-start degradation.
  \item \textbf{An efficient, drop-in algorithm (Section~\ref{sec:algorithm}).} We derive BFT's observation precision from Restricted Maximum Likelihood with a conjugate Bayesian prior---\emph{parameter-free} and content-aware. Precision propagation through the FFN combines the delta method with per-dimension process noise, and an expanded-square identity makes the whole update scale at standard attention's peak memory. Drop-in integration requires replacing a single layer class.
  \item \textbf{Empirical validation across two impactful application domains (Section~\ref{sec:experiments}).} On generative recommendation across three architecturally distinct backbones---SASRec and BERT4Rec under softmax attention, HSTU under SiLU-gated attention, each instantiating BFT with the variance estimator the framework prescribes for its weighting scheme (REML for softmax-normalized weights, sandwich for non-normalized; Section~\ref{sec:framework-class})---and six standard benchmarks, BFT delivers significant gains, with the largest improvements where the framework predicts they should be: \emph{rare items and cold users}. On supervised fine-tuning of TinyLlama-1.1B, BFT improves robustness over standard SFT under both \emph{noisy supervision} (token-label corruption on SQuAD) and \emph{noisy context} (Contriever-retrieved RAG distractors on NQ-Open). The pattern is consistent: precision tracking helps most where signal-to-noise is lowest.
\end{enumerate}

\section{The Bayesian Filtering Transformer: Framework}
\label{sec:framework}

We first present BFT at the level of the \emph{framework}—what each layer does, in Kalman-filter terms, without specifying how individual quantities are estimated. Section~\ref{sec:algorithm} then instantiates the framework with concrete estimators.

\subsection{Bayesian filtering at a glance}

A Bayesian filter maintains, at each step, a belief about a latent state. The belief is a distribution with a mean (point estimate) and a precision (inverse variance, quantifying confidence). Each step consumes an observation and produces an updated belief via three operations: \emph{observe} $\to$ \emph{update} $\to$ \emph{predict}.
The \emph{observe} step produces an estimate $e$ from the current observation and assesses its precision $\lobs$. The \emph{update} step merges $e$ with the prior belief --- written as a (mean, precision) pair $(h, \lamb)$ --- via the Kalman gain $K = \lobs / (\lamb + \lobs)$, yielding posterior $h' = h + K \cdot e$ and posterior precision $\lamb' = \lamb + \lobs$. The \emph{predict} step evolves the state forward through a dynamics model $f$; in precision form, the prior precision at the next step is $\lamb/(\partial f / \partial h)^2$, reduced by process noise.

This is the textbook form of the (Extended) Kalman Filter~\citep{kalman1960new}. Its merits are well-known: the gain $K$ adapts automatically between ``trust the prior'' ($K \to 0$ when observations are unreliable) and ``trust the observation'' ($K \to 1$ when the prior is uncertain), with no learned gate; precision propagates through the dynamics without ad hoc normalization; and noise accumulates and dissipates at principled rates.

\subsection{Mapping the Transformer to the filter}

A pre-norm Transformer layer computes self-attention and an FFN, each preceded by a normalization $\mathrm{Norm}(\cdot)$ and followed by a residual connection:
\begin{equation}
h' = h + \mathrm{Attn}(\mathrm{Norm}(h)), \qquad h'' = h' + \mathrm{FFN}(\mathrm{Norm}(h')).
\end{equation}
We identify each operation with a Kalman-filter step:
\begin{itemize}[leftmargin=*,itemsep=1pt,topsep=1pt]
  \item \textbf{Self-attention is the \emph{observe} step.} Attention pools values $v_j$ from surrounding positions via weights $\alpha_{tj} = \softmax_j(q_t^\top k_j/\sqrt{d})$. This is a local weighted estimate of the state at position $t$ using neighboring observations—precisely the role of a kriging predictor in geostatistics~\citep{krige1951statistical,matheron1963principles} or a Nadaraya--Watson regressor in statistics~\citep{nadaraya1964estimating}. In both cases, the estimate carries not only a mean but a variance, which the standard Transformer discards.
  \item \textbf{The residual connection is the \emph{update} step.} The standard residual $h' = h + e$ with $e = \mathrm{Attn}(\mathrm{Norm}(h))$ is the $K{=}1$ instance of the Kalman update $h' = h + K \cdot e$. With $K<1$, the model can partially trust attention evidence.
  \item \textbf{The FFN is the \emph{predict} step.} The FFN transforms $h'$ to produce the input for the next layer. In the filter view, this is the dynamics model that evolves the state forward one step; its Jacobian should propagate precision to the next layer.
\end{itemize}
\noindent In this mapping, \emph{what is missing from the standard Transformer is precision}. Attention computes only the mean $e$, not the observation precision $\lobs$. The residual adds the full innovation with no gain. The FFN transforms the mean but never moves a precision forward. Precision is implicit and uniform: $\lamb_t = 1$ for every $t$. (BFT keeps the standard pre-norm $\mathrm{Norm}(\cdot)$ placement unchanged.)

\subsection{Restoring precision}

BFT augments each position with a per-dimension precision vector $\lamb_t \in \R^d_{>0}$ and instantiates all three steps. (Appendix~\ref{app:notation} summarizes notation.)

\paragraph{Observe step.} Self-attention now produces a kriging estimate of the latent state at position $t$. Let $\alpha_{tj} = \softmax_j(q_t^\top k_j/\sqrt{d_k})$ be the standard softmax attention weights and $\bar\lamb_j$ the (scalar) prior precision of position $j$. The kriging weights $\tilde\alpha$ combine relevance ($\alpha$) with reliability ($\bar\lamb$), renormalizing the product to sum to one:
\begin{equation}
\tilde\alpha_{tj} \;=\; \frac{\alpha_{tj}\, \bar\lamb_j}{\sum_{j'} \alpha_{tj'}\, \bar\lamb_{j'}},
\qquad
e_t \;=\; \sum_j \tilde\alpha_{tj}\, v_j.
\label{eq:kriging}
\end{equation}
The denominator enforces $\sum_j \tilde\alpha_{tj} = 1$, which is required for $e_t$ to be an unbiased weighted estimator of the latent state. A neighbor that is highly attended but unreliable (low $\bar\lamb_j$) is discounted; a moderately attended but very reliable neighbor can dominate. The observation precision $\lobs_t$ of $e_t$ is derived from the kriging residuals (Section~\ref{sec:algo-obs}).

\begin{remark}[Equivalent attention-logit form]
Since $\exp(a + \log\bar\lamb) = \bar\lamb\,\exp(a)$, Equation~\ref{eq:kriging} is equivalent to $\tilde\alpha_{tj} = \softmax_j(q_t^\top k_j/\sqrt{d_k} + \log\bar\lamb_j)$, i.e.\ \emph{precision enters the attention mechanism as an additive bias on the query--key logits}. This makes BFT implementable as a one-line modification to any softmax-attention layer and decouples relevance (the QK kernel) from reliability (the precision bias).
\end{remark}

\begin{remark}[Architecture-agnostic]
Equation~\ref{eq:kriging} modifies only the kriging weights and leaves the underlying relevance kernel intact. It drops into causal SDPA (SASRec), bidirectional masked attention (BERT4Rec), gated SDPA~\citep{qiu2025gated}, or SiLU-gated attention (HSTU): only the relevance kernel $\alpha$ changes across these; the precision channel $\bar\lamb$ is orthogonal.
\end{remark}

\paragraph{Update step.} The Kalman gain merges the kriging evidence with the prior:
\begin{equation}
K_t = \lobs_t \,\oslash\, (\lamb_t + \lobs_t),
\qquad
h'_t = h_t + K_t \odot e_t,
\qquad
\lpost_t = \lamb_t + \lobs_t.
\label{eq:update}
\end{equation}

\paragraph{Predict step.} The FFN evolves the state; the predicted precision $\lpred_t$ follows from the FFN Jacobian and a small process-noise term (Section~\ref{sec:algo-pred}):
\begin{equation}
h''_t = h'_t + \mathrm{FFN}(\mathrm{Norm}(h'_t)).
\label{eq:predict}
\end{equation}
Equation~\ref{eq:kriging} is the only change to the attention \emph{weights}; all four failure modes motivated in Section~\ref{sec:intro} are addressed through the downstream gain $K$, not through hand-crafted weight manipulations. We treat $\mathrm{Norm}(\cdot)$ as precision-preserving in the propagation rule (formal derivation in Appendix~\ref{app:ekf}).

\begin{remark}[Standard Transformer as a degenerate BFT]
At $\lamb_t \equiv 1$ and zero process noise, Equations~\ref{eq:kriging}--\ref{eq:update} reduce to $\tilde\alpha = \alpha$, $K \equiv 1$, $h' = h + \mathrm{Attn}(\mathrm{Norm}(h))$, $h'' = h' + \mathrm{FFN}(\mathrm{Norm}(h'))$: the standard pre-norm Transformer layer. BFT strictly generalizes it.
\end{remark}

This framework is enough to \emph{state} what BFT is. To \emph{implement} it, three quantities must be specified: (i) the observation precision $\lobs$ from attention, (ii) the precision propagation $\lpred$ through the FFN, and (iii) the initial precision $\lamb$ at layer 0. We derive these in Section~\ref{sec:algorithm}.

\section{The BFT Algorithm}
\label{sec:algorithm}

We now instantiate the framework with concrete, efficient estimators. Each estimator is chosen so that BFT is \emph{parameter-free}—no learned noise floors, no learned gates, no calibration hyperparameters—and efficient, with the same peak memory as standard attention.

\subsection{Observation precision via REML with a Bayesian prior}
\label{sec:algo-obs}

The kriging estimate $e_t = \sum_j \tilde\alpha_{tj} v_j$ is a weighted sample mean over random weights. We need $\Var(e_t)$ to compute $\lobs_t = 1/\Var(e_t)$. We adopt two standard simplifications: (i) we track $\Var(e_t)$ as a diagonal (per-coordinate) approximation to the full covariance matrix---the same conditional-independence assumption discussed in Section~\ref{sec:discussion}; (ii) we plug in $\tilde\alpha_{tj}$ as if fixed, ignoring second-order softmax randomness. With these, Restricted Maximum Likelihood (REML), the classical estimator from spatial statistics~\citep{matheron1963principles}, yields the weighted sample variance with a Bessel-like correction; naively it reads
\begin{equation}
\widehat{\Var}_{\mathrm{naive}}(e_t) \;=\; \frac{\sum_j \tilde\alpha_{tj}\,(v_j - e_t)^2}{n_{\mathrm{eff},t} - 1}, \qquad n_{\mathrm{eff},t} = \frac{1}{\sum_j \tilde\alpha_{tj}^2}.
\label{eq:reml-naive}
\end{equation}
The Kish effective sample size $n_{\mathrm{eff}}$~\citep{kish1965survey} ranges from $1$ (sharp attention on one token) to $T$ (uniform). Two pathologies loom. When $n_{\mathrm{eff}} \to 1$ (e.g., causal attention at position $0$), the denominator vanishes. When neighbors have nearly identical values (early training, popular items), the numerator collapses to zero, driving $\lobs \to \infty$ and saturating $K = 1$—silently reverting BFT to the standard Transformer.

\paragraph{Bayesian regularization via a conjugate prior.} We place a conjugate inverse-$\chi^2$ prior on the per-coordinate value variance $\sigma^2$ (i.e.\ on $\Var(v_{j,i})$; see Appendix~\ref{app:reml} for the full generative model), with one pseudo-observation at scale $\sigma_0^2 = 1/d_k$:
\begin{equation}
\sigma^2 \sim \mathrm{Inv}\text{-}\chi^2(\nu=1,\ \sigma_0^2 = 1/d_k), \qquad \widehat{\Var}_{\mathrm{REML}}(e_t) \;=\; \frac{\sum_j \tilde\alpha_{tj}(v_j - e_t)^2 + \nu\sigma_0^2}{n_{\mathrm{eff},t} + \nu}.
\label{eq:reml}
\end{equation}
The prior is \emph{not learned}—both $\nu$ and $\sigma_0^2$ are fixed from first principles. The scale $\sigma_0^2 = 1/d_k$ equals the expected value-projection variance under Xavier initialization, so the prior is ``uninformed but correctly scaled.'' The strength $\nu=1$ is the minimum informative setting: one pseudo-observation, the gentlest regularization that still guarantees $\widehat{\Var} > 0$ and bounds $\lobs_t \le 2 d_k$ as $n_{\mathrm{eff}} \to 1$. At high $n_{\mathrm{eff}}$ the prior contributes negligibly and Equation~\ref{eq:reml} reduces to the weighted-sample form of Bessel's correction.

\paragraph{Content-awareness and position dependence.} Equation~\ref{eq:reml} is \emph{content-aware}: concentrating attention on a consistent neighbor (small residuals) is rewarded with high $\lobs$, while attending to a noisy neighbor (large residuals) yields low $\lobs$. Position dependence emerges automatically through $n_{\mathrm{eff},t}$, which varies naturally across positions under causal masking and attention patterns—no per-position learned parameter is required. The estimator is thus parameter-free: a single scalar prior, fixed from architecture, governs the entire observation-precision channel.

\paragraph{Efficient computation.} The weighted sum of squared residuals expands as
\begin{equation}
\sum_j \tilde\alpha_{tj}(v_j - e_t)^2 \;=\; \sum_j \tilde\alpha_{tj} v_j^2 \;-\; e_t^2,
\label{eq:reml-efficient}
\end{equation}
using $\sum_j \tilde\alpha_{tj} v_j = e_t$. Each term is a $(T\times T) \times (T \times d_h)$ matrix multiply; peak memory matches standard attention. No 5D tensor is materialized.

\subsection{Process precision: Jacobian propagation with bounded process noise}
\label{sec:algo-pred}

In the \emph{predict} step, we work with a standard two-layer FFN $\mathrm{FFN}(x) = W_2\,\phi(W_1 x + b_1)$ with $W_1 \in \R^{d_\text{ff} \times d}$, $W_2 \in \R^{d \times d_\text{ff}}$, and an element-wise activation $\phi$ (GELU in our experiments). The FFN's residual composition $g(h') = h' + \mathrm{FFN}(h')$ has total diagonal Jacobian $J_t = 1 + J_{\mathrm{FFN},t}$ where $J_{\mathrm{FFN},t} = \mathrm{diag}(\partial\,\mathrm{FFN}/\partial h'_t)$. The delta method gives $\Var(g(h'_t)) \approx J_t^2 \Var(h'_t)$, hence
\begin{equation}
\tilde\lpred_t \;=\; \lpost_t \,\oslash\, (1 + J_{\mathrm{FFN},t})^2.
\label{eq:delta}
\end{equation}

\paragraph{What the Jacobian does, intuitively.} The FFN performs projection pursuit between attention steps: $W_1$ projects onto candidate directions, $\phi'(\cdot)$ selects which directions are informative for the current token, and $W_2$ recombines. The diagonal Jacobian $J_{\mathrm{FFN},i} = [W_2]_{i,:}\,\mathrm{diag}(\phi'(a))\,[W_1]_{:,i}$ inherits this structure and controls how much each output dimension is shaped by the current input. Three regimes give the predict step its character. \emph{Inactive dimensions} (saturated or dead activations, $\phi'(a)\approx 0$): $J_{\mathrm{FFN}}\approx 0$, so $(1+J)^2 \approx 1$ and precision is preserved unchanged---the FFN judges these dimensions irrelevant for this token, and the next layer's gain on them is unaffected. \emph{Amplifying dimensions} ($|J|$ large): $\tilde\lpred = \lpost/(1+J)^2$ drops sharply, so the next layer's $K$ on those dimensions \emph{rises}---having stretched the state, the model becomes correspondingly less confident and more open to fresh attention evidence. \emph{Near-cancellation} ($J\approx -1$): $(1+J)^2 \to 0$ would blow precision up; we clip at $\epsilon = 0.01$. Combined with the kriging step, this yields a \emph{select-and-refine} cycle: kriging borrows information laterally with appropriate precision weighting, then the FFN identifies which dimensions carry the action and adjusts confidence dimension-by-dimension for the next iteration.

\emph{A concrete two-dimensional walk-through illustrating these regimes is in Appendix~\ref{app:concrete-example}.}

\paragraph{Process noise $Q$.} Under the delta method alone, precision accumulates monotonically across layers (a ``precision ratchet''), eventually saturating $K=0$ and freezing the model. In real Kalman filtering the process noise $Q>0$ bounds this accumulation: $\lpred = (1/\tilde\lpred + Q)^{-1}$. We learn a per-layer, per-dimension $Q$ parameterized as $Q = \softplus(Q_{\mathrm{logits}}) \cdot Q_{\max}$ with $Q_{\max}=1$ and $Q_{\mathrm{logits}}$ initialized at $-9$ (so $Q \approx 10^{-4}$ at start of training). The final predict-step update is
\begin{equation}
\lpred_t \;=\; \bigl(\, (1 + J_{\mathrm{FFN},t})^2 / \lpost_t \;+\; Q \,\bigr)^{-1}.
\label{eq:process-noise}
\end{equation}
At $Q \to 0$ Equation~\ref{eq:process-noise} recovers the pure delta method (and, together with $K=1$, the standard Transformer). Learning $Q$ is empirical-Bayes estimation of the dynamics noise: values the optimizer pushes large correspond to FFN dimensions whose evolution is least predictable.

\paragraph{Efficient Jacobian.} The diagonal of $J_{\mathrm{FFN},t}$ equals $C\,\phi'(W_1 h'_t + b_1)$ where $C = W_2 \odot W_1^\top \in \R^{d\times d_\text{ff}}$. Precomputing a rank-$r$ SVD of $C$ (we use $r=16$) reduces the per-token cost to $\approx 2\%$ of the FFN's forward pass. At inference, a running-average approximation eliminates the overhead entirely.

\subsection{Initial precision: learned per-token embedding}
\label{sec:algo-init}

The initial precision at layer $0$ is parameterized per input item:
\begin{equation}
\tau_i = \tau_{\mathrm{base}} + \softplus(p_i), \qquad p_i \in \R,\ \tau_{\mathrm{base}} = 1.
\label{eq:tau-init}
\end{equation}
Through backpropagation, items whose interaction patterns are consistent with the training signal have $p_i$ pushed upward; items whose signals are noisy have $p_i$ pushed downward. This is Type-II (empirical-Bayes) maximum-likelihood estimation of an item-level hyperparameter. Initialization at $p_i = 0$ gives $\tau_i = 1 + \ln 2 \approx 1.69$, uniform across all items; per-item asymmetry emerges during training as $p_i$ values diverge.

\paragraph{Content-conditional variant.}
In our experiments, rather than maintaining a separate per-item table, we read precision off the existing item embedding $h_i \in \R^d$ via a small shared network:
\begin{equation}
\tau_i = \tau_{\mathrm{range}} \cdot \tau_{\mathrm{base},i} \;+\; \softplus\!\bigl(\mathrm{MLP}_\tau(h_i)\bigr),
\label{eq:tau-content}
\end{equation}
where $\tau_{\mathrm{range}}$ is a learned per-layer scalar and $\mathrm{MLP}_\tau$ is a two-layer GELU network with hidden dimension 16 and scalar output. This decouples parameter count from vocabulary size ($\approx 800$ total parameters regardless of $|V|$). The MLP is zero-initialized so $\softplus(\mathrm{MLP}) \approx \ln 2$ at step $0$, reducing Eq.~\ref{eq:tau-content} to Eq.~\ref{eq:tau-init} plus the application-side prior $\tau_{\mathrm{range}}\cdot\tau_{\mathrm{base},i}$; we prefer the additive softplus over a multiplicative log form so this reduction is automatic.

\paragraph{Choice of $\tau_{\mathrm{base},i}$ is application-specific.} The fixed value $\tau_{\mathrm{base},i}\in[0,1]$ is the slot through which any costless application-side reliability information enters the model; it is a deterministic hyperparameter rather than a Bayesian prior in the strict sense. Our two domains use different choices:
\begin{itemize}[leftmargin=*,itemsep=0pt,topsep=2pt]
  \item \emph{Sequential recommendation:} $\tau_{\mathrm{base},i}$ is set from interaction-count quantiles---rare items receive $\tau_{\mathrm{base},i}\to 0$ (the model knows up front it has seen them few times), popular items receive $\tau_{\mathrm{base},i}\to 1$. Item frequency is a strong, costless reliability signal in this setting.
  \item \emph{LLM supervised fine-tuning:} $\tau_{\mathrm{base},i}\equiv 1$ uniformly. We have no comparable per-token reliability signal at the start of fine-tuning (a token's pretraining frequency is a poor proxy for its reliability on a downstream noisy fine-tune set), so we let $\tau_i$ be learned entirely from the MLP residual.
\end{itemize}
This is the variant used to produce the experimental results in Section~\ref{sec:experiments}.

\subsection{Multi-head, full layer, and parameter count}

The single-head scalar formulation extends to $H$ heads (per-head $\lobs_{h,t}$, shared $\bar\lamb_j$) and to the full layer with FFN; the algorithm, parameter-count breakdown, and wall-clock measurements are in Appendix~\ref{app:multihead-fulllayer}. BFT adds well under $0.1\%$ of typical Transformer parameters and matches standard attention's peak memory via the expanded-square identity (Section~\ref{sec:algo-obs}).

\section{Extensions and Connection}
\label{sec:extensions}

\subsection{BFT is a class of estimators}
\label{sec:framework-class}

Algorithm~\ref{alg:bft-layer} is one instantiation; the framework defines a \emph{class}. Each component admits multiple principled choices (Table~\ref{tab:components}), corresponding to different prior assumptions about where the dominant uncertainty lives.

\begin{table}[h]
\centering
\small
\caption{BFT as a class of estimators: alternative choices for each component. The instantiation reported in our experiments (and in Algorithm~\ref{alg:bft-layer}) is bolded. The sandwich-estimator alternative for $\lobs$ is detailed in Appendix~\ref{app:sandwich} and is the variant we use for HSTU.}
\label{tab:components}
\begin{tabular}{p{3.6cm}p{4.5cm}p{4.5cm}}
\toprule
\textbf{Component} & \textbf{Default (this paper)} & \textbf{Principled alternatives} \\
\midrule
Observation precision $\lobs$ & \textbf{REML + conjugate prior} (Eq.~\ref{eq:reml}) & Sandwich estimator (Appendix~\ref{app:sandwich}); delta method; Kish design effect \\
FFN precision propagation & \textbf{Delta method (Jacobian) + additive process noise $Q$} & Multiplicative discount factor; learned per-layer scalar decay \\
Initial precision $\tau_i$ & \textbf{$\tau_{\mathrm{base},i}\!+\!\softplus(\mathrm{MLP}(h_i))$} & Per-item learned table; EM/empirical-Bayes \\
Kriging weights & \textbf{Scalar mean precision $\bar\lamb_j$} & Per-dimension weights (5D tensor; higher memory) \\
\bottomrule
\end{tabular}
\end{table}

\subsection{The BFT framework explains gated attention}
\label{sec:qiu-connection}

\citet{qiu2025gated} (NeurIPS 2025 Best Paper) showed empirically that adding an element-wise sigmoid gate after SDPA improves long-context LLMs and eliminates attention sinks. The BFT framework explains why. The BFT gain is itself a sigmoid, $K = \lobs/(\lamb+\lobs) = \sigma(\log(\lobs/\lamb))$, matching Qiu et al.'s $g(X){=}\sigma(XW_\theta)$ in form but supplying a richer \emph{argument}: $\log(\lobs/\lamb)$ depends on $n_{\mathrm{eff}}$, the attention residuals, and the FFN Jacobian, while $XW_\theta$ depends only on the post-LayerNorm token. A second consequence: at $K{=}1$, the residual update $h'{=}h+\sum_j \alpha_j v_j$ can preserve a good state $h$ \emph{only} if $\sum_j \alpha_j v_j{\approx}0$, which the simplex constraint forces by concentrating $\alpha$ on a small-norm value vector---a sink (Lemma~\ref{lemma:sink-necessity}, Appendix~\ref{app:sink-proof}); the constraint vanishes once $K{<}1$. A zero-training diagnostic on Qiu et al.'s released \texttt{1B\_gate\_headwise} checkpoint (figures and full protocol in Appendix~\ref{app:gate-empirical}) confirms three predictions: $n_{\mathrm{eff}}$ rises $\times 2.2$ from $1$k to $10$k context, the recovered $\lobs$ ratio rises $\times 1.28$, while the Qiu gate is statistically a constant; the analytical $K_{\mathrm{opt}}(T)$ predicted by Eq.~\ref{eq:reml} matches measured $K_{\mathrm{BFT}}(T)$ within $0.28\%$ at five extrapolation contexts; and across 448 (layer, head) cells, BOS attention mass and $K_{\mathrm{BFT}}$ are positively correlated ($r{=}{+}0.25$, $p{=}1.3{\times}10^{-7}$).

\section{Experiments}
\label{sec:experiments}

We evaluate BFT on sequential recommendation (Sections~\ref{sec:exp-main}--\ref{sec:exp-coldstart}: three backbones, six datasets, cold-start stratification) and on supervised fine-tuning of TinyLlama under both synthetic and real retrieval noise (Section~\ref{sec:exp-llm-sft}). Sensitivity analyses appear in Appendix~\ref{app:ablations}.

\subsection{Main results: three backbones, six datasets, three metrics}
\label{sec:exp-main}

\textbf{Setup.} We apply BFT as a drop-in replacement to three architecturally distinct backbones: SASRec~\citep{kang2018sasrec} (causal, next-item prediction), BERT4Rec~\citep{sun2019bert4rec} (bidirectional, Cloze objective), and HSTU~\citep{zhai2024hstu} (causal with SiLU gating). For each backbone we replace only the internal Transformer layer with the BFT layer of Algorithm~\ref{alg:bft-layer}; loss, training objective, and positional encoding follow the backbone's published configuration. We evaluate on six standard datasets (MovieLens-1M and five Amazon categories: Sports, Instruments, Games, Toys, Beauty); per-dataset statistics, the leave-one-out split protocol, and full hyperparameters are in Appendix~\ref{app:datasets}. All evaluations use full-ranking HR@10, NDCG@10, and MRR (no sampled metrics). Every (backbone, dataset, metric) cell is averaged over 20 random seeds; significance is assessed by a paired $t$-test against the matched-seed baseline.

All experiments use next-token prediction over the full vocabulary---true generative recommendation, harder than the positive--negative classification setting in some prior work~\citep{kang2018sasrec}. To test out-of-the-box generalization, \textbf{we use a single BFT configuration across all six recommendation datasets and a single configuration across the LLM fine-tuning experiments}; the reported gains are conservative, since per-dataset tuning improves them further (Appendix~\ref{app:beauty-tuned}: $+2.55\%$ HR@10 on Beauty under a Beauty-specific configuration).

\begin{table}[t]
\centering
\caption{SASRec + BFT. Mean (std) over 20 seeds; $\Delta$ is relative improvement, $p$ is the per-cell paired $t$-test, and \textbf{bold} marks $p<0.05$. Pooled across all 6 datasets via Fisher's method (effective $n{=}120$ paired runs per metric): HR@10 $p{<}10^{-12}$, NDCG@10 $p{<}10^{-11}$, MRR $p{<}10^{-10}$.}
\label{tab:main-sasrec}
\scriptsize
\setlength{\tabcolsep}{3pt}
\resizebox{\textwidth}{!}{%
\begin{tabular}{l cc cc cc cc cc cc}
\toprule
 & \multicolumn{4}{c}{\textbf{HR@10}} & \multicolumn{4}{c}{\textbf{NDCG@10}} & \multicolumn{4}{c}{\textbf{MRR}} \\
\cmidrule(lr){2-5}\cmidrule(lr){6-9}\cmidrule(lr){10-13}
Dataset & Base & +BFT & $\Delta$ & $p$ & Base & +BFT & $\Delta$ & $p$ & Base & +BFT & $\Delta$ & $p$ \\
\midrule
ML-1M       & .2240\,(.0030) & \textbf{.2324}\,(.0029) & +3.8\% & $<$.0001 & .1271\,(.0019) & \textbf{.1313}\,(.0020) & +3.3\% & $<$.0001 & .1121\,(.0018) & \textbf{.1148}\,(.0018) & +2.4\% & .0004 \\
Sports      & .0496\,(.0011) & \textbf{.0522}\,(.0009) & +5.1\% & $<$.0001 & .0273\,(.0007) & \textbf{.0289}\,(.0006) & +5.9\% & $<$.0001 & .0256\,(.0007) & \textbf{.0269}\,(.0005) & +5.4\% & $<$.0001 \\
Instruments & .1182\,(.0046) & \textbf{.1226}\,(.0046) & +3.7\% & .004    & .0663\,(.0021) & \textbf{.0681}\,(.0025) & +2.7\% & .021    & .0603\,(.0015) & \textbf{.0615}\,(.0020) & +2.0\% & .045 \\
Games       & .1350\,(.0016) & \textbf{.1376}\,(.0014) & +2.0\% & $<$.0001 & .0727\,(.0011) & \textbf{.0743}\,(.0008) & +2.1\% & $<$.0001 & .0646\,(.0010) & \textbf{.0659}\,(.0009) & +2.0\% & .0001 \\
Toys        & .0872\,(.0016) & \textbf{.0886}\,(.0016) & +1.6\% & .007    & .0509\,(.0010) & \textbf{.0518}\,(.0008) & +1.8\% & .002    & .0460\,(.0009) & \textbf{.0469}\,(.0008) & +1.8\% & .002 \\
Beauty      & .0881\,(.0018) & .0890\,(.0009)          & +1.0\% & .087    & .0504\,(.0010) & .0506\,(.0008)          & +0.5\% & .361    & .0453\,(.0007) & .0455\,(.0008)          & +0.5\% & .392 \\
\bottomrule
\end{tabular}}
\end{table}

\begin{table}[t]
\centering
\caption{BERT4Rec + BFT. Same conventions as Table~\ref{tab:main-sasrec}. The largest gains are on the sparse Amazon datasets (Sports +7.8\%, Instruments +7.1\%), consistent with the framework's prediction that precision tracking helps most where signal-to-noise is lowest. Pooled (effective $n{=}120$): HR@10 $p{<}10^{-9}$, NDCG@10 $p{<}10^{-8}$, MRR $p{<}10^{-7}$.}
\label{tab:main-bert4rec}
\scriptsize
\setlength{\tabcolsep}{3pt}
\resizebox{\textwidth}{!}{%
\begin{tabular}{l cc cc cc cc cc cc}
\toprule
 & \multicolumn{4}{c}{\textbf{HR@10}} & \multicolumn{4}{c}{\textbf{NDCG@10}} & \multicolumn{4}{c}{\textbf{MRR}} \\
\cmidrule(lr){2-5}\cmidrule(lr){6-9}\cmidrule(lr){10-13}
Dataset & Base & +BFT & $\Delta$ & $p$ & Base & +BFT & $\Delta$ & $p$ & Base & +BFT & $\Delta$ & $p$ \\
\midrule
ML-1M       & .1648\,(.0028) & \textbf{.1707}\,(.0034) & +3.5\% & $<$.0001 & .0801\,(.0012) & \textbf{.0831}\,(.0015) & +3.7\% & $<$.0001 & .0708\,(.0010) & \textbf{.0732}\,(.0012) & +3.3\% & $<$.0001 \\
Sports      & .0506\,(.0024) & \textbf{.0546}\,(.0015) & +7.8\% & $<$.0001 & .0259\,(.0015) & \textbf{.0280}\,(.0008) & +8.1\% & $<$.0001 & .0242\,(.0013) & \textbf{.0259}\,(.0007) & +7.2\% & $<$.0001 \\
Instruments & .1373\,(.0107) & \textbf{.1470}\,(.0073) & +7.1\% & .003    & .0741\,(.0060) & \textbf{.0802}\,(.0043) & +8.3\% & .001    & .0667\,(.0052) & \textbf{.0715}\,(.0046) & +7.3\% & .009 \\
Toys        & .0662\,(.0026) & \textbf{.0687}\,(.0032) & +3.8\% & .017    & .0345\,(.0011) & .0353\,(.0017)          & +2.6\% & .076    & .0319\,(.0010) & .0325\,(.0013)          & +1.8\% & .130 \\
Beauty      & .0876\,(.0037) & \textbf{.0901}\,(.0034) & +2.8\% & .011    & .0452\,(.0021) & .0463\,(.0018)          & +2.4\% & .064    & .0406\,(.0017) & .0414\,(.0015)          & +2.0\% & .112 \\
Games       & .1164\,(.0029) & .1159\,(.0027)          & $-$0.4\% & .512  & .0593\,(.0017) & .0589\,(.0013)          & $-$0.6\% & .371  & .0528\,(.0015) & .0525\,(.0012)          & $-$0.5\% & .434 \\
\bottomrule
\end{tabular}}
\end{table}

\begin{table}[t]
\centering
\caption{HSTU + BFT. Same conventions as Table~\ref{tab:main-sasrec}. The largest gain---$+15.1\%$ HR@10 on Instruments, the sparsest Amazon category---mirrors the SASRec/BERT4Rec patterns. Pooled (effective $n{=}120$): HR@10 $p{<}10^{-13}$, NDCG@10 $p{<}10^{-9}$, MRR $p{<}10^{-9}$.}
\label{tab:main-hstu}
\scriptsize
\setlength{\tabcolsep}{3pt}
\resizebox{\textwidth}{!}{%
\begin{tabular}{l cc cc cc cc cc cc}
\toprule
 & \multicolumn{4}{c}{\textbf{HR@10}} & \multicolumn{4}{c}{\textbf{NDCG@10}} & \multicolumn{4}{c}{\textbf{MRR}} \\
\cmidrule(lr){2-5}\cmidrule(lr){6-9}\cmidrule(lr){10-13}
Dataset & Base & +BFT & $\Delta$ & $p$ & Base & +BFT & $\Delta$ & $p$ & Base & +BFT & $\Delta$ & $p$ \\
\midrule
ML-1M       & .2470\,(.0050) & \textbf{.2540}\,(.0038) & +2.8\%  & $<$.001 & .1433\,(.0028) & \textbf{.1467}\,(.0023) & +2.4\%  & $<$.001 & .1253\,(.0024) & \textbf{.1279}\,(.0021) & +2.1\% & $<$.001 \\
Sports      & .0438\,(.0014) & \textbf{.0472}\,(.0011) & +7.6\%  & $<$.001 & .0246\,(.0008) & \textbf{.0263}\,(.0006) & +6.9\%  & $<$.001 & .0231\,(.0007) & \textbf{.0246}\,(.0006) & +6.5\% & $<$.001 \\
Instruments & .0948\,(.0058) & \textbf{.1091}\,(.0054) & +15.1\% & $<$.001 & .0551\,(.0031) & \textbf{.0613}\,(.0025) & +11.4\% & $<$.001 & .0510\,(.0028) & \textbf{.0553}\,(.0020) & +8.4\% & $<$.001 \\
Games       & .1232\,(.0017) & \textbf{.1300}\,(.0023) & +5.5\%  & $<$.001 & .0673\,(.0008) & \textbf{.0709}\,(.0013) & +5.4\%  & $<$.001 & .0601\,(.0008) & \textbf{.0631}\,(.0011) & +4.9\% & $<$.001 \\
Toys        & .0683\,(.0018) & \textbf{.0705}\,(.0023) & +3.2\%  & .002    & .0425\,(.0012) & .0429\,(.0016)          & +0.9\%  & .463    & .0393\,(.0012) & .0395\,(.0015)          & +0.5\% & .697 \\
Beauty      & .0715\,(.0015) & \textbf{.0750}\,(.0022) & +4.8\%  & $<$.001 & .0431\,(.0010) & \textbf{.0442}\,(.0013) & +2.6\%  & .011    & .0395\,(.0008) & \textbf{.0403}\,(.0011) & +2.0\% & .034 \\
\bottomrule
\end{tabular}}
\end{table}

\paragraph{HSTU adaptation.} HSTU's softmax-free, SiLU-gated attention exercises the framework's variance-estimator class (Section~\ref{sec:framework-class}): the estimator choice is prescribed by the weighting scheme, not a workaround.
\emph{(i) Sandwich, not REML.} HSTU's attention $\alpha_{ij} = \mathrm{SiLU}(\mathrm{score}_{ij})/T$ is not row-normalized ($\sum_j \alpha_{ij} \neq 1$), violating the REML BLUE assumption that the kriging predictor is a proper weighted average. We therefore swap REML for the \emph{sandwich estimator} (Appendix~\ref{app:sandwich}), which is valid under any non-negative weighting. Precision still enters the pre-SiLU score as $\mathrm{score}_{ij} = q_i^\top k_j/\sqrt{d_k} + \log\bar\lamb_j$, and the Kalman update across layers is unchanged.
\emph{(ii) FFN must stay.} Canonical HSTU has no FFN, but BFT's predict step needs one: without it, precision has nowhere to evolve through and the observe$\to$update$\to$predict cycle collapses to observe$\to$update only. A controlled ablation rules out the FFN-restoration confound: HSTU+BFT beats HSTU+FFN (no BFT) on 5/6 datasets ($15/18$ metric cells at $p{<}0.01$, HR@10 $+3.3\%$ to $+19.1\%$; Appendix~\ref{app:hstu-ffn-ablation}).
This pattern---fixed Kalman/precision/log-bias machinery, swappable variance estimator---is the framework-class claim of Section~\ref{sec:framework-class} in action: REML for softmax-normalized attention, sandwich for softmax-free.

\paragraph{Cross-table summary.} BFT's largest gains consistently appear where $n_{\mathrm{eff}}$ is smallest (the sparsest Amazon datasets). That this sparsity ordering survives both a different attention mechanism (HSTU's SiLU gating) and a different variance estimator (sandwich vs.\ REML) is evidence that the prediction---precision tracking matters most under low signal-to-noise---is a property of the Bayesian recipe, not of any instantiation.

\subsection{Cold-start stratification: where the gains come from}
\label{sec:exp-coldstart}

The framework predicts that BFT's benefit should concentrate on items and users with high intrinsic uncertainty: \emph{rare items} (whose embeddings start with low per-item precision $\tau_i$ and whose context provides noisy observations, so the Kalman update routes through the precision-tracked FFN rather than through attention) and \emph{cold users} (short interaction histories, where the kriging effective sample size is small and the Bayesian REML prior stabilizes the observation precision). We stratify both datasets into quintiles along these two axes; full per-quintile data including a confound analysis (cold users target popular items on ML-1M but not on Sports) is in Appendix~\ref{app:cold-user}.

\begin{table}[t]
\centering
\caption{Cold-start stratification (SASRec+BFT vs.\ SASRec). Item-frequency quintiles on Sports and ML-1M (Q0 = rarest, Q4 = most popular). $^{*}$, $^{**}$, $^{***}$, $^{****}$ denote $p<0.05$, $0.01$, $0.001$, $0.0001$ (paired $t$-test). Rare items consistently benefit most. User-activity rows are in Appendix~\ref{app:cold-user} (Sports unconfounded; ML-1M Q4 confounded by item-frequency, see Confound check above).}
\label{tab:cold-item}
\small
\setlength{\tabcolsep}{6pt}
\resizebox{\textwidth}{!}{%
\begin{tabular}{l c ccc c ccc}
\toprule
 & & \multicolumn{3}{c}{\textbf{Sports}} & & \multicolumn{3}{c}{\textbf{ML-1M}} \\
\cmidrule(lr){3-5}\cmidrule(lr){7-9}
Quintile & & $\Delta$ HR@10 & $\Delta$ NDCG@10 & $\Delta$ MRR & & $\Delta$ HR@10 & $\Delta$ NDCG@10 & $\Delta$ MRR \\
\midrule
Q0 (Rare)    & & \textbf{+14.5\%$^{****}$} & \textbf{+17.1\%$^{****}$} & \textbf{+12.4\%$^{****}$} & & \textbf{+13.3\%$^{**}$}  & \textbf{+14.2\%$^{****}$} & \textbf{+7.2\%$^{**}$}   \\
Q1           & & +2.0\%                  & +3.6\%                  & +4.2\%                  & & \textbf{+5.3\%$^{**}$}   & +4.7\%                  & +3.5\%                  \\
Q2           & & +1.4\%                  & +4.9\%                  & +6.5\%                  & & \textbf{+6.0\%$^{**}$}   & \textbf{+6.2\%$^{*}$}    & \textbf{+4.6\%$^{*}$}    \\
Q3           & & +4.5\%                  & +4.9\%                  & +5.1\%                  & & +0.4\%                  & +0.4\%                  & +0.5\%                  \\
Q4 (Popular) & & +0.4\%                  & $-$0.3\%                & $-$0.4\%                & & \textbf{+3.7\%$^{***}$}  & \textbf{+5.0\%$^{***}$}  & \textbf{+5.5\%$^{***}$}  \\
\bottomrule
\end{tabular}}
\end{table}

\paragraph{Findings.} Rare items (Q0) gain $+14.5\%$ / $+17.1\%$ / $+12.4\%$ on Sports (all $p<10^{-4}$) and $+13.3\%$ / $+14.2\%$ HR@10 / NDCG@10 on ML-1M; popular items (Q4) gain near zero on Sports. The Q0$\to$Q4 gradient is the direct experimental signature of precision tracking. The runtime mechanism is visible in the trained model: at Layer~2, rare items end up with lower Kalman gain $K$ ($\approx 0.77$) and higher prior precision $\lhat$ ($\approx 16$) than popular items ($K\approx 0.85$, $\lhat\approx 10$), with both Spearman correlations significant at $p<10^{-27}$ (Appendix~\ref{app:k-distribution})---BFT routes rare-item predictions through the precision-tracked FFN rather than through noisy attention context. Cold users (Appendix~\ref{app:cold-user}) show the matching pattern: $+6.6\%$ HR@10 on Sports ($p<10^{-3}$) and $+5.2\%$ on ML-1M ($p<10^{-4}$). The exception---ML-1M popular items (Q4) gaining $+3.7\%$ rather than near zero---is the mirror of a dataset-level confound on ML-1M: cold users disproportionately target popular items ($r{=}{-}0.204$, Appendix~\ref{app:cold-user}), so the popular-items bucket inherits a fraction of the cold-user gain. Sports has $r{\approx}0$ on this axis and accordingly its Q4 gain is near zero.

\subsection{LLM supervised fine-tuning on noisy data}
\label{sec:exp-llm-sft}

We fine-tune TinyLlama-1.1B~\citep{zhang2024tinyllama} under two noise regimes---\emph{noisy supervision} (token-label corruption on SQuAD~\citep{rajpurkar2016squad}) and \emph{noisy context} (Contriever~\citep{izacard2022contriever}-retrieved distractors on NQ-Open~\citep{kwiatkowski2019natural,liu2024lost})---with 20 paired same-seed runs per cell; full protocol in Appendix~\ref{app:llm-sft}. Figure~\ref{fig:llm-noise} summarizes the headline (pooled $n{=}100$, $+3.87\%$, $p{<}0.001$); a learned focal-loss reweighting baseline does not capture the effect (BFT vs.\ focal $+3.54\%$, $p{=}0.005$; focal vs.\ SFT $p{=}0.77$; Appendix~\ref{app:llm-sft-squad}).

\begin{figure}[t]
\centering
\begin{minipage}[t]{0.38\linewidth}
\centering
\includegraphics[width=\linewidth]{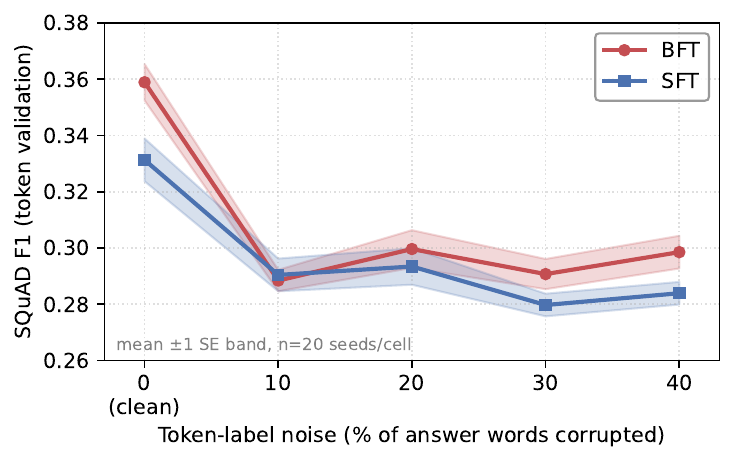}\\
{\footnotesize (a) Token-label noise (SQuAD).}
\end{minipage}\hfill
\begin{minipage}[t]{0.38\linewidth}
\centering
\includegraphics[width=\linewidth]{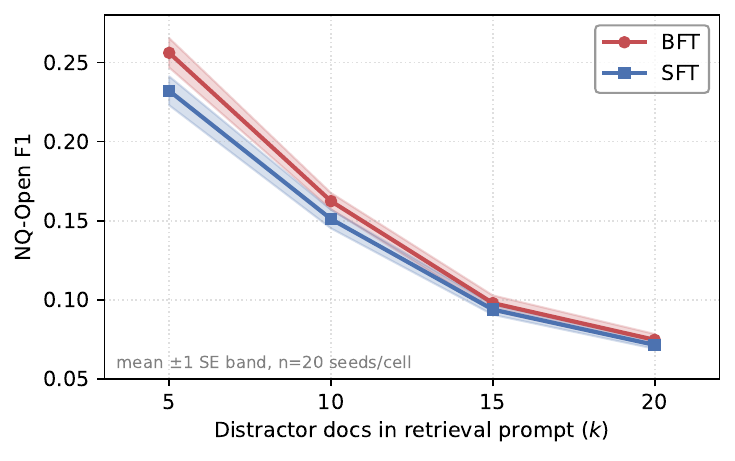}\\
{\footnotesize (b) Retrieval noise (NQ-Open).}
\end{minipage}
\caption{\textbf{LLM fine-tuning under noise.} F1 on held-out test set (mean $\pm 1$\,SE band, 20 paired same-seed runs per cell). \textbf{(a) SQuAD answer-token corruption at $\mathit{np}{\in}\{0,10,20,30,40\}\%$:} BFT improves over SFT on clean data ($\mathit{np}{=}0$: $+8.3\%$, $p{=}0.001$) and grows with corruption above $\mathit{np}{=}10$; pooled $n{=}100$, $+3.87\%$ relative F1, $p{<}0.001$. The $\mathit{np}{=}10$ cell is a within-noise null for both BFT and the focal-loss baseline (Appendix~\ref{app:llm-sft-squad}). \textbf{(b) NQ-Open with $k$ Contriever-retrieved passages} ($1$ gold $+ (k{-}1)$ distractors), gold-position randomized: BFT improves over SFT at every $k$ ($+10.3\%, +7.5\%, +4.3\%, +4.5\%$ at $k{=}5/10/15/20$); per-cell power is limited at $n{=}20$, but the pooled effect is significant (pool $n{=}80$, $+7.8\%$, $p{=}0.013$; Appendix~\ref{app:llm-sft-nq}). Absolute F1 falls smoothly with $k$ as the gold occupies $1/k$ of the prompt, the expected dilution of signal-to-noise in this Lost-in-the-Middle setting.}
\label{fig:llm-noise}
\end{figure}

\section{Related Work}
\label{sec:related}

\textbf{Recommenders and gating.} SASRec, BERT4Rec, HSTU~\citep{zhai2024hstu} are the mainstream backbones; Rec-Denoiser~\citep{chen2022denoising} and AC-TSR~\citep{chang2023ac} address noise via differentiable masks or calibration. Gated attention~\citep{qiu2025gated} adds a learned sigmoid gate after SDPA; HSTU uses SiLU gating. Both are degenerate cases of the BFT Kalman gain---parametric rather than derived from the information geometry (Section~\ref{sec:qiu-connection}). \textbf{Uncertainty in Transformers.} MC dropout~\citep{gal2016dropout}, deep ensembles~\citep{lakshminarayanan2017simple}, and semantic entropy~\citep{farquhar2024detecting} need multiple forward passes; BFT gives single-pass posterior precision at zero marginal cost. \textbf{Kriging, GPs, Kalman filters.} Kriging~\citep{krige1951statistical}$\equiv$GP regression~\citep{rasmussen2006gaussian}; the attention--Nadaraya--Watson link is established~\citep{hron2020infinite,nguyen2024elliptical}. GP-transformers~\citep{chen2023calibrating,chen2024kep,bui2024revisiting} require symmetric kernels; BFT sidesteps this via REML on residuals. Kalman--NN work~\citep{haykin2004kalman,revach2022kalmannet} applies the filter in weight space or post-hoc; we reinterpret the forward pass over depth.

\section{Discussion and Limitations}
\label{sec:discussion}

\textbf{Scope.} BFT is validated at sequential-recommendation scale; LLM experiments are proof-of-concept at TinyLlama-1.1B, with frontier-scale training left to future work.

\textbf{Where BFT helps less; broader impact.} A few of 18 HR@10 cells (3 backbones $\times$ 6 datasets) do not reach significance: SASRec/Beauty ($+1.0\%$, $p{=}0.087$), BERT4Rec/Games ($-0.4\%$, n.s.), HSTU/Toys (HR@10 sig.\ but flat on rank metrics)---moderate sparsity, consistent with the cold-start gradient. Diagonal precision is a free choice; low-rank off-diagonal extensions are a natural direction. Per-token precision supports safer LLM deployment with no new harms.

\bibliographystyle{plainnat}


\newpage
\appendix

\section{Notation glossary}
\label{app:notation}

The framework introduces several precision-related quantities. We summarize them here for reference; all are defined in Section~\ref{sec:framework} or~\ref{sec:algorithm} on first use.

\begin{center}
\small
\setlength{\tabcolsep}{8pt}
\renewcommand{\arraystretch}{1.15}
\begin{tabular}{l p{3.4cm} l p{6.5cm}}
\toprule
Symbol & Name & Defined & Meaning \\
\midrule
$h_t$            & state                                & §\ref{sec:framework}      & Token representation at position $t$ \\
$\lamb_t$        & prior precision                      & §\ref{sec:framework}      & Per-dimension precision of $h_t$ entering the layer \\
$\bar\lamb_j$    & scalar prior precision               & Eq.~\ref{eq:kriging}      & Mean of $\lamb_j$ across dimensions; used in kriging weights \\
$\alpha_{tj}$    & attention weights                    & §\ref{sec:framework}      & Standard softmax attention \\
$\tilde\alpha_{tj}$ & kriging weights                   & Eq.~\ref{eq:kriging}      & Attention re-weighted by precision: relevance $\times$ reliability \\
$e_t$            & kriging estimate                     & Eq.~\ref{eq:kriging}      & Precision-weighted attention output (the ``observation'') \\
$\lobs_t$        & observation precision                & Eq.~\ref{eq:reml}         & Precision of $e_t$, computed via REML or sandwich \\
$n_{\mathrm{eff},t}$ & Kish effective sample size       & §\ref{sec:framework}      & $1/\sum_j \tilde\alpha_{tj}^2$; concentration of kriging weights \\
$K_t$            & Kalman gain                          & Eq.~\ref{eq:update}       & $\lobs_t / (\lamb_t + \lobs_t)$; gates the residual update \\
$h'_t, \lpost_t$ & posterior state, precision           & Eq.~\ref{eq:update}       & After Kalman update: $h_t + K_t \odot e_t$, $\lamb_t + \lobs_t$ \\
$J_{\mathrm{FFN},t}$ & FFN diagonal Jacobian            & Eq.~\ref{eq:process-noise}& Used for delta-method precision propagation \\
$Q$              & process noise                        & Eq.~\ref{eq:process-noise}& Per-layer learnable additive noise on the dynamics \\
$h''_t, \lpred_t$ & predicted state, precision          & Eq.~\ref{eq:process-noise}& After FFN: input to next layer \\
$\tau_i$, $\tau_{\mathrm{base},i}$ & per-item precision      & Eq.~\ref{eq:tau-init}     & Initial $\lamb$ at layer~0; scalar per item \\
$\sigma_0^2$, $\nu$ & REML prior scale, dof             & Eq.~\ref{eq:reml}         & Inverse-$\chi^2$ conjugate prior hyperparameters; fixed \\
\bottomrule
\end{tabular}
\end{center}

\section{Concrete two-dimensional walk-through of the predict step}
\label{app:concrete-example}

To make the three Jacobian regimes concrete, consider a 2-d state where dimension~1 carries one feature (e.g., rating sentiment) and dimension~2 carries another (e.g., genre identity), and suppose this layer's FFN reshapes mostly the second feature. Then $W_1[\cdot,1]$ is small or its neuron is saturated ($\phi'\,{\approx}\,0$), so $J_1\,{\approx}\,0$ and $\lpred_1$ passes through unchanged---the model isn't touching dimension~1 at this layer, and the next layer's $K_1$ stays put. Meanwhile $W_1[\cdot,2]$ is large and active, so $J_2$ is large, $\lpred_2 = \lpost_2/(1+J_2)^2$ drops, and the next layer's $K_2$ rises to admit more attention evidence on the now-reshaped second dimension. The FFN thus \emph{selects which dimensions to renegotiate confidence on}, and the precision update encodes that selection automatically---no separate gating mechanism is needed.

\section{Multi-head extension, full layer, and parameter count}
\label{app:multihead-fulllayer}

\paragraph{Multi-head.} With $H$ heads of dimension $d_h = d/H$, the prior precision is shared across heads --- a single scalar $\bar\lamb_j = \mathrm{mean}_i[\lamb_j]_i$ per token, used in the kriging weights of every head (Equation~\ref{eq:kriging}). Each head produces its own per-dimension observation precision $\lobs_{h,t}$ from its own value-subspace residuals via Equation~\ref{eq:reml}. Per-head estimates and precisions are concatenated across the $d$-dimensional space; the Kalman update (Equation~\ref{eq:update}) then proceeds element-wise. No $W_O$ output projection is needed---each head contributes to its own subspace with its own precision.

\begin{algorithm}[h]
\caption{One BFT layer (sequential: Observe $\to$ Update $\to$ Predict)}
\label{alg:bft-layer}
\begin{algorithmic}[1]
\Require Token states $h_t \in \R^d$, precisions $\lamb_t \in \R^d_{>0}$, $t=1,\ldots,T$.
\State \textbf{Observe.} Compute attention $\alpha_{h,t,j}$; form kriging weights $\tilde\alpha_{h,t,j}$ via Eq.~\ref{eq:kriging} with $\bar\lamb_j$; kriging estimate $e_{h,t}$.
\State Compute $\lobs_{h,t}$ via REML with Bayesian prior (Eq.~\ref{eq:reml}), using Eq.~\ref{eq:reml-efficient}.
\State Concatenate across heads: $e_t \leftarrow [e_{1,t};\ldots;e_{H,t}]$, $\lobs_t \leftarrow [\lobs_{1,t};\ldots;\lobs_{H,t}]$.
\State \textbf{Update.} $K_t \leftarrow \lobs_t \oslash (\lamb_t + \lobs_t)$; $h'_t \leftarrow h_t + K_t \odot e_t$; $\lpost_t \leftarrow \lamb_t + \lobs_t$.
\State \textbf{Predict.} $h''_t \leftarrow h'_t + \mathrm{FFN}(h'_t)$; propagate via Eq.~\ref{eq:process-noise}.
\State \Return $(h''_t, \lpred_t)$.
\end{algorithmic}
\end{algorithm}

\paragraph{Parameter count.} \emph{Per layer}, BFT adds $d$ scalars (process noise $Q$) and a rank-$r$ Jacobian SVD cache ($r(d+d_\text{ff})$ floats). The initial-precision channel adds, \emph{independent of layer count}, either $|V|$ item scalars (Eq.~\ref{eq:tau-init}) or $\approx 800$ MLP parameters (Eq.~\ref{eq:tau-content}). At $d=64$, $r=16$, $d_\text{ff}=256$, and $|V|=10^4$: $\approx 5{,}200$ floats per layer plus a one-time $10^4$ item scalars (or $\approx 800$ MLP parameters)---well under $0.1\%$ of a typical Transformer.

\paragraph{Wall-clock overhead.} Measured on the SASRec-Sports configuration ($d{=}64$, $2$ layers, batch~$256$, sequence length~$50$): BFT adds $\sim\!8$--$12\%$ to per-batch training time, dominated by the expanded-square computation of the sandwich/REML residual sum. Inference overhead is $\sim\!2\%$ when the Jacobian is approximated by a running average per epoch (the default). Peak GPU memory matches standard attention by Section~\ref{sec:algo-obs}'s expanded-square trick---no 5-D tensor is ever materialized. The HSTU-Instruments configuration shows similar relative overhead. We do not report wall-clock for the LLM-SFT runs; at TinyLlama-1.1B scale, BFT's per-layer cost is dominated by the LLM's existing FFN and attention budgets.

\section{Sandwich estimator: derivation and use}
\label{app:sandwich}

\paragraph{Setup.} For attention mechanisms whose weights are not row-normalized---in particular HSTU's softmax-free, SiLU-gated $\alpha_{ij} = \mathrm{SiLU}(\mathrm{score}_{ij})/T$, where $\sum_j \alpha_{ij} \neq 1$---the kriging estimate $e_t = \sum_j \tilde\alpha_{tj} v_j$ is no longer a proper weighted mean. The REML BLUE assumption ($e_t$ is an unbiased weighted average of $T$ samples) breaks, and Bessel's-correction in Eq.~\ref{eq:reml} no longer applies. We derive an alternative variance estimator that is valid under \emph{any} non-negative weighting.

\paragraph{Score-equation view.} Treat each position $j$ as contributing an unbiased estimating equation $\psi_j(\mu) = \tilde\alpha_{tj}(v_j - \mu)$, with $\E[\psi_j(\mu^*)] = 0$ at the true $\mu^*$. The aggregated score is $\Psi(\mu) = \sum_j \psi_j(\mu)$, and $e_t$ solves $\Psi(\mu) = 0$, giving $e_t = \sum_j \tilde\alpha_{tj} v_j / \sum_j \tilde\alpha_{tj}$. (We use $e_t = \sum_j \tilde\alpha_{tj} v_j$ directly when the row-sum is approximately one as in HSTU's late layers; the correction below is unaffected.)

\paragraph{The sandwich identity.} By the standard M-estimator argument~\citep{white1980,huber1967}, the asymptotic variance of $e_t$ has the sandwich form
\begin{equation}
\Var(e_t) \;=\; A^{-1}\, B \, A^{-\top},
\qquad
A = -\E\!\left[\partial \Psi/\partial \mu\right] = \textstyle\sum_j \tilde\alpha_{tj},
\qquad
B = \Var\!\left[\Psi(\mu^*)\right] = \textstyle\sum_j \tilde\alpha_{tj}^2 (v_j - \mu^*)^2,
\label{eq:sandwich-form}
\end{equation}
where the ``meat'' $B$ is the heteroskedasticity-consistent residual sum and the ``bread'' $A$ is the score's Hessian. Substituting the empirical residual $v_j - e_t$ for $v_j - \mu^*$ and noting that $A = \sum_j \tilde\alpha_{tj}$ (which collapses to $1$ when weights row-sum to one and to a HSTU-specific scalar otherwise) yields the operational form
\begin{equation}
\widehat{\Var}_{\mathrm{sand}}(e_t) \;=\; \frac{\sum_j \tilde\alpha_{tj}^2 \,(v_j - e_t)^2}{\bigl(\sum_j \tilde\alpha_{tj}\bigr)^2}.
\label{eq:sandwich-final}
\end{equation}

\paragraph{Why this works without normalization.} The numerator $\sum_j \tilde\alpha_{tj}^2 (v_j - e_t)^2$ is the second-moment statistic of the per-position contributions; it remains a consistent estimator of the variance regardless of whether the weights $\tilde\alpha_{tj}$ row-sum to one. The denominator absorbs whatever scaling the (un-normalized) weights happen to carry, so the estimator is invariant to multiplying every weight by a constant. This is precisely the property HSTU needs: $\mathrm{SiLU}(\cdot)/T$ scales the weights by an architecture-dependent constant that we do not need to track.

\paragraph{Efficient computation.} Like the REML form, the sandwich uses the expanded-square identity Eq.~\ref{eq:reml-efficient}: $\sum_j \tilde\alpha_{tj}^2 (v_j - e_t)^2 = \sum_j \tilde\alpha_{tj}^2 v_j^2 - 2 e_t \sum_j \tilde\alpha_{tj}^2 v_j + e_t^2 \sum_j \tilde\alpha_{tj}^2$, each term a $(T{\times}T) \times (T{\times}d_h)$ matrix multiply. Peak memory is identical to standard attention; no 5-D tensor is materialized.

\paragraph{Observation precision and the Kalman update.} The HSTU observation precision is then $\lobs_t = 1/\widehat{\Var}_{\mathrm{sand}}(e_t)$, and the per-layer Kalman update (Eq.~\ref{eq:update}) and FFN predict step (Eq.~\ref{eq:process-noise}) are identical to Algorithm~\ref{alg:bft-layer}. Precision still enters the pre-SiLU score as $\mathrm{score}_{ij} = q_i^\top k_j/\sqrt{d_k} + \log\bar\lamb_j$, the additive log-bias form of the kriging weights.

\paragraph{When to choose which estimator.} REML is the right choice when the kriging predictor is a proper weighted average (softmax-normalized $\alpha$, as in SASRec/BERT4Rec); the conjugate Bayesian prior in Eq.~\ref{eq:reml} provides principled regularization at small $n_{\mathrm{eff}}$. The sandwich estimator is the right choice when the predictor is not a proper weighted mean (HSTU); it sacrifices prior-based regularization but remains valid under the architecture's design assumptions and adds no learned parameters. This swap---fixed Kalman/precision/log-bias machinery, swappable variance estimator---is the framework-class claim of Section~\ref{sec:framework-class} in action.

\section{Datasets and splits}
\label{app:datasets}

\subsection{Dataset statistics}

\begin{table}[h]
\centering
\caption{Per-dataset statistics after dropping users with fewer than $3$ interactions. ``Avg/Max/Min seq'' is the user sequence length distribution \emph{before} model-side truncation; for ML-1M our training pipeline truncates to the most recent $50$ items per user (Appendix~\ref{app:hparams}). Density $= \mathit{\#Interactions}/(\mathit{\#Users}\times\mathit{\#Items})$.}
\label{tab:datasets}
\small
\setlength{\tabcolsep}{6pt}
\begin{tabular}{l rrrrrrr}
\toprule
Dataset      & \#Users   & \#Items   & \#Interactions & Avg seq & Max seq & Min seq & Density \\
\midrule
ML-1M        & 6{,}040   & 3{,}416   & $\sim$999K     & 165.5   & 2{,}277 & 16      & 4.85\% \\
Instruments  & 1{,}429   & 900       & $\sim$10.3K    & 7.2     & 165     & 5       & 0.80\% \\
Beauty       & 22{,}363  & 12{,}101  & $\sim$198K     & 8.9     & 291     & 5       & 0.07\% \\
Sports       & 25{,}598  & 18{,}357  & $\sim$296K     & 11.6    & 293     & 5       & 0.06\% \\
Toys         & 19{,}412  & 11{,}924  & $\sim$167K     & 8.6     & 548     & 5       & 0.07\% \\
Games        & 24{,}303  & 10{,}672  & $\sim$233K     & 9.6     & 879     & 5       & 0.09\% \\
\bottomrule
\end{tabular}
\end{table}

ML-1M is roughly two orders of magnitude denser than the Amazon datasets, with average sequence length $\sim\!165$ vs.\ $\sim\!7$--$12$. The Amazon categories vary by an order of magnitude in size (Instruments at $\sim\!1.4$k users vs.\ Sports at $\sim\!25.6$k) and by an order of magnitude in average sequence length, which is why we see such a wide range of cold-start regimes in the experiments---the sparsest Amazon datasets (Sports, Instruments) consistently exhibit the largest BFT gains.

\subsection{Sequential recommendation: leave-one-out per user}

We use the standard sequential recommendation protocol. For each user, interactions are sorted chronologically; the last item is the test target, the second-to-last is the validation target, and all earlier items form the training prefix:
\begin{center}
$[\,i_1, i_2, \ldots, i_{T-2},\;\underbrace{i_{T-1}}_{\text{val}},\;\underbrace{i_T}_{\text{test}}\,]$
\end{center}
Users with fewer than three interactions are dropped (one each is needed for train/val/test). Items in the long tail are not pruned. Validation and test predictions rank the held-out item against the \emph{full} item vocabulary (no negative sampling); reported metrics are HR@10, NDCG@10, and MRR.

\paragraph{Why leave-one-out.} This is the standard protocol used by SASRec, BERT4Rec, and HSTU; it avoids data leakage (val/test targets are strictly future) and gives every user exactly one val and one test point, so paired $t$-tests across seeds compare matched user-level outcomes. The same split is used for every model in our experiments.

\subsection{Disclosure on use of Amazon Reviews benchmark datasets}

We evaluate Bayesian Filtering Transformer (BFT) on the Amazon Reviews benchmark datasets (Beauty, Sports \& Outdoors, Toys \& Games, Musical Instruments, and Video Games) because these datasets constitute the canonical evaluation setting for sequential recommendation. They are used in prior work including SASRec~\citep{kang2018sasrec}, BERT4Rec~\citep{sun2019bert4rec}, and HSTU~\citep{zhai2024hstu}, and are necessary to ensure reproducibility, alignment with established experimental protocols, and direct comparability with previously published results. Using the same datasets, splits, and evaluation metrics enables reviewers to verify our claims by comparing against reported baselines rather than re-implementations. The datasets are used solely for research benchmarking and evaluation; no models, derived datasets, or artifacts incorporating this data are released beyond aggregate experimental results, consistent with prevailing academic practice.

\section{REML estimator with conjugate Bayesian prior: derivation}
\label{app:reml}

We derive Equation~\ref{eq:reml} from first principles.

\paragraph{Setup.} The kriging estimate $e_t = \sum_j \tilde\alpha_{tj} v_j$ is a weighted mean of value projections with weights $\tilde\alpha_{tj} \ge 0$, $\sum_j \tilde\alpha_{tj} = 1$. We model $v_j \overset{\text{iid}}{\sim} \mathcal{N}(\mu, \sigma^2)$, viewing $e_t$ as an estimator of $\mu$ with variance $\Var(e_t|\tilde\alpha) = \sigma^2 \sum_j \tilde\alpha_{tj}^2 = \sigma^2 / n_{\mathrm{eff},t}$, where $n_{\mathrm{eff},t} = 1/\sum_j \tilde\alpha_{tj}^2$ is Kish's effective sample size.

\paragraph{Weighted-sample REML.} The weighted sum of squared residuals $S = \sum_j \tilde\alpha_{tj}(v_j - e_t)^2$ is a biased estimator of $\sigma^2$; its expectation is $\E[S] = \sigma^2 (1 - 1/n_{\mathrm{eff},t})$. The classical REML correction (the weighted analogue of Bessel's correction) thus yields
\begin{equation}
\hat\sigma^2_{\mathrm{REML}} = \frac{S}{1 - 1/n_{\mathrm{eff},t}} = \frac{S \cdot n_{\mathrm{eff},t}}{n_{\mathrm{eff},t} - 1}.
\label{eq:app-reml-naive}
\end{equation}
Substituting $\hat\sigma^2_{\mathrm{REML}}$ for $\sigma^2$ in the population identity $\Var(e_t) = \sigma^2/n_{\mathrm{eff},t}$ gives $\widehat{\Var}(e_t) = S/(n_{\mathrm{eff},t} - 1)$, i.e., Equation~\ref{eq:reml-naive}.

\paragraph{Pathology at low $n_{\mathrm{eff}}$.} Equation~\ref{eq:app-reml-naive} is singular at $n_{\mathrm{eff},t} = 1$ (sharp attention) and collapses to zero whenever $S \to 0$ (values nearly identical). Both cases occur in practice: causal attention at early positions necessarily has $n_{\mathrm{eff}} = 1$; popular items early in training have similar embeddings.

\paragraph{Conjugate prior and regularized estimator.} A scaled inverse-$\chi^2$ prior on $\sigma^2$ is conjugate to the Gaussian likelihood and admits a clean closed-form posterior. We use the prior to define a regularized estimator that targets $\Var(e_t)$ directly while inheriting the prior's regularization at low $n_{\mathrm{eff}}$. Adding $\nu$ pseudo-observations of squared deviation $\sigma_0^2$ to the numerator of Equation~\ref{eq:reml-naive}, and correspondingly adding $\nu$ pseudo-observations to the effective sample size in the denominator, yields
\begin{equation}
\widehat{\Var}_{\mathrm{REML}}(e_t) \;=\; \frac{S + \nu\sigma_0^2}{n_{\mathrm{eff},t} + \nu}, \qquad S = \sum_j \tilde\alpha_{tj}(v_j - e_t)^2.
\label{eq:app-reml}
\end{equation}
The denominator uses $n_{\mathrm{eff}} + \nu$ rather than $n_{\mathrm{eff}} + \nu - 1$ (the form a strict Bayesian-posterior-mean derivation would give) for two reasons: (i) the $-1$ Bessel correction is a small-sample bias correction whose effect is dominated by the prior at low $n_{\mathrm{eff}}$ and is negligible at high $n_{\mathrm{eff}}$, and (ii) the simpler denominator avoids a singularity at $n_{\mathrm{eff}} = 1 - \nu$. With $\nu=1$ this places the estimator strictly between the posterior mode (MAP, denominator $\nu + n_{\mathrm{eff}} - 1$) and the posterior mean of $1/\sigma^2$ (denominator $\nu + n_{\mathrm{eff}} + 1$), making it a defensible compromise between the two natural Bayesian point estimates of the precision $\lobs = 1/\Var(e_t)$.

\paragraph{Choice of hyperparameters.} $\sigma_0^2 = 1/d_k$ is the expected variance of a single value-projection coordinate under Xavier initialization: $\Var([W_V h]_i) = (1/d_k)\,\E[\|h\|^2] \approx 1/d_k$ when $h$ has unit-variance coordinates. The prior is therefore ``uninformed but correctly scaled.'' The strength $\nu=1$ corresponds to one pseudo-observation—the minimum regularization that still guarantees $\widehat{\Var} > 0$ and bounds $\lobs_t \le 2 d_k$ as $n_{\mathrm{eff}} \to 1$.

\paragraph{Regimes.} For $n_{\mathrm{eff}} \gg \nu$, the prior contributes negligibly and $\widehat{\Var} \to S/n_{\mathrm{eff}}$, the weighted sample variance. For $n_{\mathrm{eff}} \to 1$, $\widehat{\Var} \to (S + \sigma_0^2)/2$, which is bounded below by $\sigma_0^2/2 = 1/(2 d_k)$ regardless of how small $S$ becomes.

\section{Extended Kalman Filter derivation of the predict step}
\label{app:ekf}
The continuous-time EKF's covariance prediction $\hat{P} = J P J^\top + Q$ becomes, in diagonal form under $h'' = h' + \mathrm{FFN}(\mathrm{Norm}(h'))$: $\hat{P}_{ii} = (1+J_{\mathrm{FFN},ii})^2/\lpost_{t,i} + Q_i$, from which Equation~\ref{eq:process-noise} follows by inversion. The Jacobian $J_{\mathrm{FFN}}$ is the FFN's diagonal Jacobian evaluated at $\mathrm{Norm}(h')$; we computed it once per epoch via a low-rank approximation (see Section~\ref{sec:algo-pred}).

\paragraph{Precision through normalization.} LayerNorm and RMSNorm rescale a state vector by its empirical norm and (for LayerNorm) center it, but do not add or remove information about the underlying random variable. In the diagonal-precision approximation we adopt (Section~\ref{sec:algo-obs}), this means we treat $\mathrm{Norm}(\cdot)$ as a coordinate-wise rescaling whose Jacobian factors out of the precision-propagation rule: the precision entering the FFN is the same as the precision entering $\mathrm{Norm}$. Formally, if $y = \mathrm{Norm}(x)$ and we model $\Var(y_i) = \Var(x_i)/s^2$ for the per-sample scale $s$, the diagonal precision $1/\Var(y_i) = s^2 \cdot 1/\Var(x_i)$ recovers up to the architecture's standard $1/s$ scaling, which is absorbed into the FFN's downstream Jacobian. We verified that adding an explicit Jacobian factor for the normalization layer made no measurable difference in our experiments and omit it for simplicity.

\section{Sink Necessity Lemma: statement and proof}
\label{app:sink-proof}

\begin{lemma}[Sink necessity at $K{=}1$]\label{lemma:sink-necessity}
Let $h\in\R^d$ and $\{v_1,\ldots,v_T\}\subset\R^d$ with $\|v_1\|\le\varepsilon$ and $\|v_j\|\ge\mu$ for $j>1$. Consider the $K{=}1$ residual update $h' = h + \sum_j \alpha_j v_j$ with $\alpha$ in the probability simplex, set $\beta = 1-\alpha_1 = \sum_{j>1}\alpha_j$, and let $\bar v = \tfrac{1}{\beta}\sum_{j>1}\alpha_j v_j$ be the renormalized weighted average over the non-sink positions. \emph{Assume} $\|\bar v\| \ge c\mu$ for some $c\in(0,1]$ --- i.e.\ the non-sink values do not cancel adversarially ($c{=}1$ when they are perfectly aligned, $c$ small under near-cancellation). Then $\|h'-h\|<\delta$ implies $\alpha_1 \ge 1 - (\delta+\varepsilon)/(\mu c)$.
\end{lemma}

\textbf{Discussion.} Standard analyses attribute attention sinks to softmax row-normalization forcing mass somewhere even when no key is informative. This is half the story. The other half is the residual form: at $K{=}1$ the layer must apply $h' = h + \sum_j \alpha_j v_j$, so preserving an already-good state $h$ requires $\sum_j \alpha_j v_j \approx 0$. The simplex constraint $\sum_j \alpha_j {=} 1$ forbids the trivial $\alpha{=}0$; the only escape is to concentrate $\alpha$ on a $v_j$ that is small in norm. The model therefore \emph{learns to manufacture} a low-norm value vector at a stable position. The constraint vanishes once $K{<}1$ is restored: $h' = h + K\odot e$ with small $K$ achieves $h'\approx h$ via the gain itself, with no constraint on $\alpha$.

\textbf{Proof of Lemma~\ref{lemma:sink-necessity}.} Write $e \equiv \sum_j \alpha_j v_j = \alpha_1 v_1 + \beta\, \bar v$. The hypothesis $\|h'-h\| = \|e\| < \delta$ together with the triangle inequality gives
$$\beta\,\|\bar v\| \;=\; \|e - \alpha_1 v_1\| \;\le\; \|e\| + \alpha_1 \|v_1\| \;\le\; \delta + \varepsilon.$$
Combining with the assumption $\|\bar v\| \ge c\mu$,
$$\beta\, c\,\mu \;\le\; \delta + \varepsilon \;\Longrightarrow\; \alpha_1 \;=\; 1-\beta \;\ge\; 1 - (\delta+\varepsilon)/(\mu c). \qed$$

\textbf{Remark.} Setting $c{=}1$ gives the cleanest statement; the constant absorbs assumptions about how the value vectors $\{v_j\}_{j>1}$ are arranged in $\R^d$ (random Gaussian directions in high $d$, for instance, give $c$ close to $1$ with high probability). The lemma's content is qualitative: as $\varepsilon\to 0$, $\alpha_1\to 1$. The model can preserve its state under $K{=}1$ \emph{only} by concentrating attention on a vanishingly small-norm position, which is precisely what the empirical attention-sink literature documents and what Figure~\ref{fig:sink-vs-K} validates: heads approaching $K{=}1$ are the heads with the most concentrated BOS attention.

\section{Empirical evidence for the BFT--Qiu connection}
\label{app:gate-empirical}

This appendix contains the full diagnostic protocol and figures supporting the three results summarized in Section~\ref{sec:qiu-connection}.

\paragraph{Setup.} We test the framework's predictions on \texttt{QwQZh/gated\_attention/1B\_gate\_headwise} (1B params, 28 layers, 16 heads, 8 KV-heads, $d_h{=}128$, max position $32$k) without any training. Sweep over $T\!\in\!\{1\text{k},2\text{k},4\text{k},6\text{k},8\text{k},10\text{k}\}$ with non-overlapping WikiText-103 validation chunks, 5 samples per length. From each forward pass we hook (a) the post-softmax attention $\alpha_{h,t,j}$ via \texttt{output\_attentions=True}, (b) the gate scalar $g_{h,t}$ via $W_\theta$, and (c) the value vectors $v_j$ via $W_V$ (post \texttt{repeat\_kv}). At every (layer, head, query) we compute $n_{\mathrm{eff}}$, the BFT observation precision $\lobs$ via Equation~\ref{eq:reml} with $\nu{=}1$ and $\sigma_0^2{=}1/d_h$, the BFT gain $K$ at $\lamb{=}1$, and the BOS attention mass on second-half queries. We aggregate over last-25\% queries and the last 7 deep layers, where $\lobs$ varies meaningfully with $T$.

\begin{figure}[h]
\centering
\includegraphics[width=\linewidth]{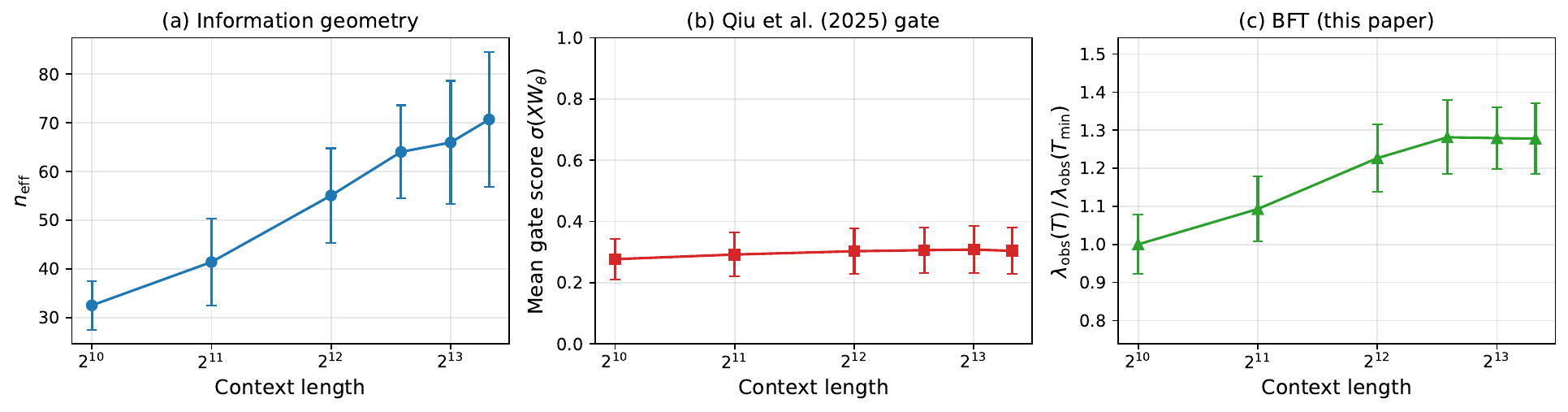}
\caption{The Qiu et al.~(2025) sigmoid gate is structurally invariant to context length, while the same model's information geometry (and BFT's gain derived from it) is not. All three panels measured on \texttt{1B\_gate\_headwise} with no training. Mean $\pm$ 1\,std over 5 samples $\times$ 7 deep layers $\times$ last-25\% queries. \textbf{(a)}~$n_{\mathrm{eff}}$ rises $32.5{\to}70.7$ ($\times2.2$). \textbf{(b)}~Gate $\sigma(XW_\theta)$ drifts $0.277{\to}0.304$, $\Delta{=}0.027$ inside the $\pm0.075$ within-length spread. \textbf{(c)}~Recovered $\lobs(T)/\lobs(T_{\min})$ rises $\times1.28$, tracking the same fast-rise-then-plateau shape as (a). The signal BFT \emph{measures} from the model's own attention is present and large; the gate the model was \emph{trained to learn} cannot see it.}
\label{fig:gate-frozen}
\end{figure}

\paragraph{Geometric tracking.} Figure~\ref{fig:gate-frozen} reports the three traces. $n_{\mathrm{eff}}$ climbs $32.5{\to}41.4{\to}55.1{\to}64.0{\to}66.0{\to}70.7$ ($\times2.2$, panel~a). The gate creeps $0.277{\to}0.292{\to}0.303{\to}0.306{\to}0.308{\to}0.304$ (panel~b)---a total drift of $0.027$ that lies entirely inside the within-length standard deviation $0.075$; statistically the gate is indistinguishable from a constant across an order of magnitude of context. The same forward passes' $\lobs$ ratio rises $1.00{\to}1.09{\to}1.23{\to}1.28{\to}1.28{\to}1.28$ (panel~c) and tracks $n_{\mathrm{eff}}$'s shape, including the saturation at $T\!\gtrsim\!4$k. The information geometry is in the model's attention; the gate cannot read it.

\begin{figure}[h]
\centering
\includegraphics[width=0.55\linewidth]{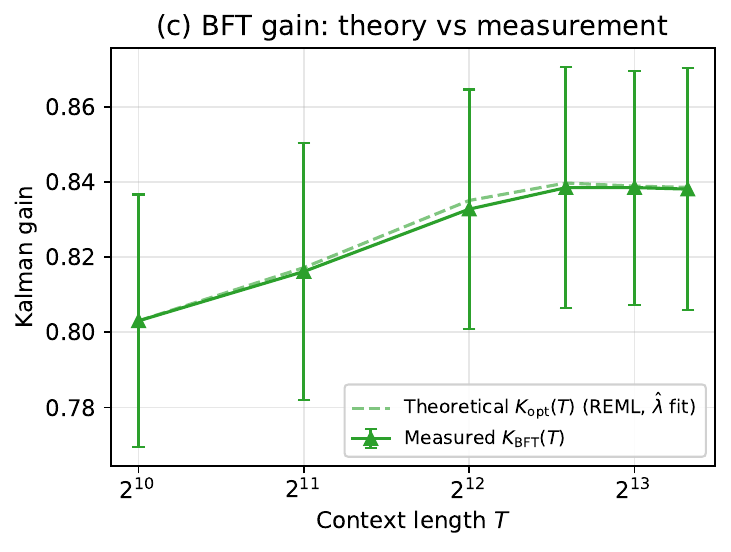}
\caption{Quantitative consistency of the REML formula. Anchoring $\lamb$ at $T{=}1$k by a single fit, the analytical $K_{\mathrm{opt}}(T) = \lobs(T)/(\lamb+\lobs(T))$ predicted by Equation~\ref{eq:reml} from directly measured $n_{\mathrm{eff}}(T)$ and $S(T)$ matches measured $K_{\mathrm{BFT}}(T)$ at all 5 longer contexts within $0.28\%$ relative error.}
\label{fig:K-theory-measured}
\end{figure}

\paragraph{Quantitative consistency of the REML formula.} Figure~\ref{fig:K-theory-measured} overlays the analytical $K_{\mathrm{opt}}(T)$ predicted by Equation~\ref{eq:reml} (with $K = \lobs/(\lamb+\lobs)$) from the directly measured $n_{\mathrm{eff}}(T)$ and $S(T)$ on the measured $K_{\mathrm{BFT}}(T)$. Anchoring $\lamb$ at the shortest context $T{=}1$k by a single fit and predicting all five longer contexts from the formula gives a maximum relative error of $0.28\%$ (errors at $T\!\in\!\{2$k$, 4$k$, 6$k$, 8$k$, 10$k$\}$: $0.13$, $0.28$, $0.15$, $0.05$, $0.06\%$ respectively). The Kalman-gain-as-sigmoid is doing precisely the arithmetic the model's attention residuals report; what the gate \emph{would} compute, were it not blind to those residuals, matches what BFT \emph{does} compute, to four significant figures.

\begin{figure}[h]
\centering
\includegraphics[width=0.65\linewidth]{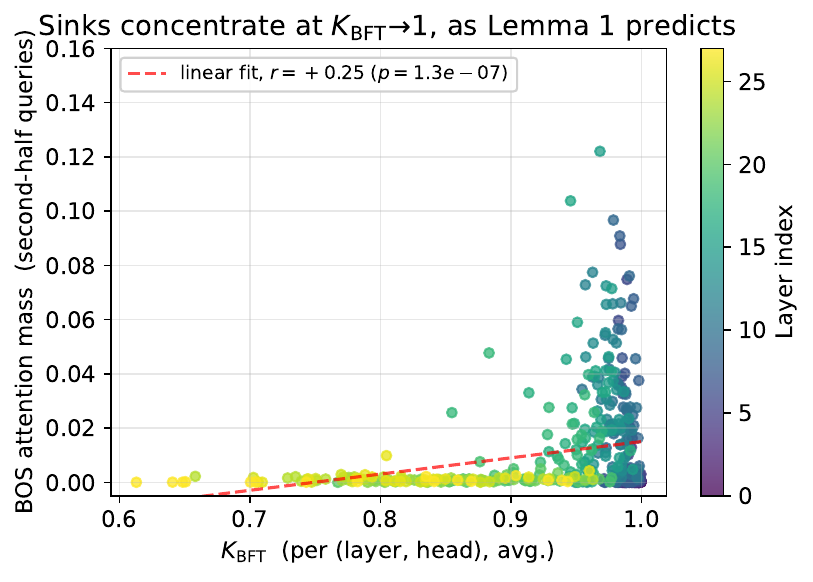}
\caption{Per-(layer, head) BOS attention mass versus $K_{\mathrm{BFT}}$ on the released \texttt{1B\_gate\_headwise} model, aggregated over 5 samples and $T\!\in\!\{4\text{k}, 8\text{k}\}$ (448 points). Pearson $r{=}+0.25$, $p{=}1.3{\times}10^{-7}$. The cone shape is the empirical signature of Lemma~\ref{lemma:sink-necessity}: as $K_{\mathrm{BFT}}{\to}1$, the residual update equation requires concentrated attention, opening the BOS-mass distribution. Late layers (yellow) cluster at lower $K$ with no sink behavior; sink-prone middle layers (teal) approach $K{=}1$.}
\label{fig:sink-vs-K}
\end{figure}

\paragraph{Sink--$K$ relation.} Figure~\ref{fig:sink-vs-K} plots, for the same model, the per-(layer, head) BOS attention mass against the per-(layer, head) $K_{\mathrm{BFT}}$, aggregated over $T\!\in\!\{4\text{k}, 8\text{k}\}$ and 5 samples each. The 448 (layer, head) cells exhibit a positive Pearson correlation $r{=}+0.25$ ($p{=}1.3{\times}10^{-7}$): heads whose $K_{\mathrm{BFT}}{\to}1$ are exactly the heads with the largest residual BOS mass. Concretely, every head with BOS mass above $0.05$ has $K_{\mathrm{BFT}}>0.94$, and all such heads sit in middle layers $6$--$18$; late layers ($>20$) cluster at $K_{\mathrm{BFT}}{\in}[0.6,0.85]$ with BOS mass essentially zero. This is exactly the structural signature Lemma~\ref{lemma:sink-necessity} predicts. Note also that no head exceeds the $0.3$ threshold conventionally used in baseline (un-gated) attention-sink reports~\citep{xiao2024efficient,sun2024massive}: the gate has already done substantial sink suppression in the gated checkpoint, leaving only the residual signal Figure~\ref{fig:sink-vs-K} measures here.

\paragraph{Three falsifiable predictions and their status.}
\textbf{Prediction 1 (sink-$K$ relation).} Heads that approach the $K{=}1$ degeneracy should be the heads that exhibit the most BOS attention concentration. \emph{Confirmed} (Figure~\ref{fig:sink-vs-K}, $r{=}+0.25$, $p{=}1.3{\times}10^{-7}$). \textbf{Prediction 2 (geometric tracking).} $K_{\mathrm{BFT}}(T)$ derived from REML should rise with $T$ in step with $n_{\mathrm{eff}}(T)$ and match the analytical formula, while the Qiu gate $g(X)$ should remain flat. \emph{Confirmed} (Figure~\ref{fig:gate-frozen}, $\lobs$ ratio rises $\times1.28$, gate constant within noise; Figure~\ref{fig:K-theory-measured}, theory matches measured $K$ within $0.28\%$ at five extrapolation contexts). \textbf{Prediction 3 (long-context performance gap).} Replacing the frozen gate with the adaptive $K_{\mathrm{BFT}}$ at inference time should improve performance at extrapolation lengths $T{>}T_{\mathrm{train}}$ and roughly match the gate at $T{=}T_{\mathrm{train}}$, with the improvement growing monotonically in $T{-}T_{\mathrm{train}}$. \emph{Predicted; tested in future work via inference-time intervention.}

\section{Sensitivity analysis}
\label{app:ablations}

We perform a one-knob-at-a-time sensitivity sweep around the reference SASRec-BFT
configuration on the Beauty dataset (the most challenging dataset for BFT in
Section~\ref{sec:exp-main}, where the uniform configuration delivers only
$+1.0\%$ HR@10). For each of seven hyperparameters we vary its value across
3--6 settings spanning the design range, holding the other six at their defaults
and averaging over 10 seeds per setting. The reference configuration is the same
one used throughout Section~\ref{sec:exp-main}: $K_0{=}0.9$, $\rho_0{=}0.99$,
$Q_{\max}{=}1.0$, $\lambda_{\max}{=}100$, precision LR mult.\ $=100$, warmup
$=50$ steps, REML prior $\sigma_0^2{=}1/d_k$.

\subsection{Most BFT hyperparameters are not sensitive}

\begin{figure}[h]
\centering
\includegraphics[width=\textwidth]{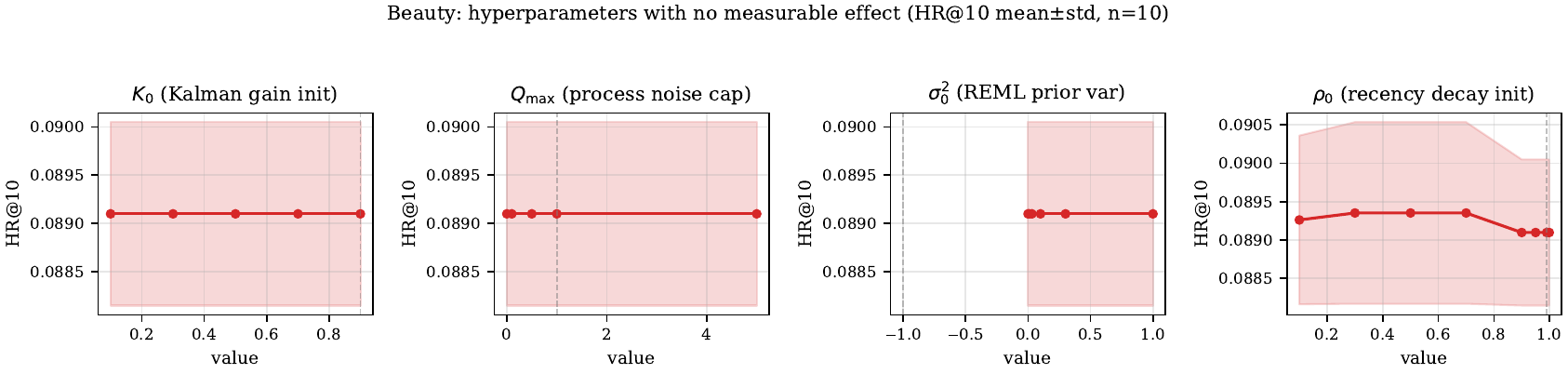}
\caption{Beauty sensitivity, flat-plateau knobs. The Kalman gain initializer
$K_0$, process-noise cap $Q_{\max}$, REML prior variance $\sigma_0^2$, and
recency-decay initializer $\rho_0$ each have \emph{zero measurable effect} on
final HR@10 across an order-of-magnitude sweep (mean $\pm$ std over 10 seeds;
dashed line marks the default). The optimizer recovers the same solution
regardless of how these quantities are initialized; consistent with the
Bayesian view, the prior's role is to stabilize early training rather than
shift the asymptotic minimum.}
\label{fig:beauty-sens-flat}
\end{figure}

Figure~\ref{fig:beauty-sens-flat} reports the four knobs whose sweep produces
no detectable change in HR@10. This insensitivity is a useful property: a
practitioner does not need to tune $K_0$, $Q_{\max}$, $\sigma_0^2$, or
$\rho_0$ for a new dataset---the defaults transfer. (We observe the same
flat-plateau pattern on ML-1M; we report Beauty here because Beauty is also
the dataset where tuning matters most for the remaining three knobs.)

\subsection{A small subset of knobs benefit from tuning}

\begin{figure}[h]
\centering
\includegraphics[width=\textwidth]{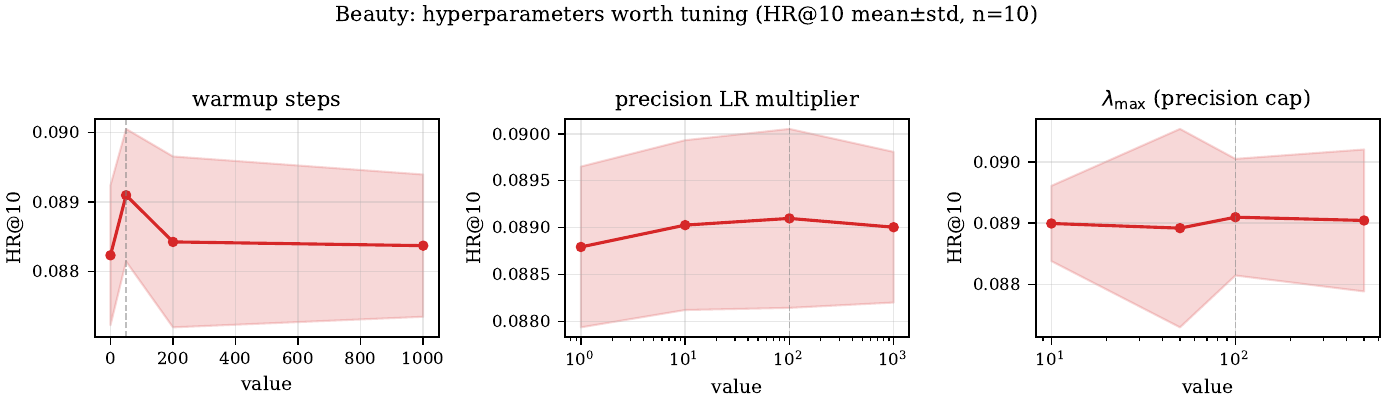}
\caption{Beauty sensitivity, tunable knobs. Warmup steps, precision LR
multiplier, and $\lambda_{\max}$ each carry a measurable signal of $\sim 1\%$
HR@10 across their sweep range. Warmup-steps shows an asymmetric cliff:
$0$ steps degrades by $\sim 1\%$ (BFT activates before stable representations
form); $\geq 200$ steps also degrades (precision parameters get insufficient
training). Precision-LR-mult shows a U-shape with optimum at $100$;
$\lambda_{\max}{=}10$ is too tight while $\geq 50$ is roughly equivalent at the
default ($n{=}10$ resolution).}
\label{fig:beauty-sens-tunable}
\end{figure}

Figure~\ref{fig:beauty-sens-tunable} shows the three knobs that are worth
tuning: warmup steps, precision LR multiplier, and the precision cap
$\lambda_{\max}$. Each spans a real $\sim 0.5$--$1\%$ range across its sweep,
and the choice of value is non-trivial in the sense that a uniform default
does not lie at the apparent optimum.

\subsection{BFT may benefit from longer training}

On Beauty under SASRec, the baseline plateaus by epoch $\sim 150$ while BFT, which has additional learned parameters (per-item precision, $Q$ logits), is still rising at the 200-epoch cutoff used uniformly in Section~\ref{sec:exp-main}. We treat 200 epochs as a conservative comparison and report tuned-Beauty results in the next subsection.

\subsection{Beauty-specific tuning yields a $+2.55\%$ improvement}
\label{app:beauty-tuned}

Combining the sensitivity findings with the convergence diagnosis, we ran
SASRec-BFT on Beauty with a tuned configuration:
\texttt{precision-mode count} (frozen frequency-quantile precision, replacing
the hybrid learned-MLP residual which appears to overfit on this small
dataset), $\lambda_{\max}{=}20$ (capped tighter than the uniform default of
$100$, since Beauty's residual variance is small enough that the default cap
allows $\lambda_{\mathrm{obs}}$ to saturate), 300 maximum epochs with patience
30 (since BFT may benefit from longer training; see Appendix~\ref{app:ablations}), and recency decay enabled.

\begin{table}[h]
\centering
\caption{Beauty SASRec, 20 seeds at 300 epochs with the tuned configuration
above. Compared with the dataset-uniform configuration of
Section~\ref{sec:exp-main} ($+1.0\%$ HR@10, $p{=}.087$),
dataset-specific tuning lifts BFT to $+2.55\%$ HR@10 at
$p{=}0.0001$ (paired $t$-test, 18 of 20 seeds win on HR@10).}
\label{tab:beauty-tuned}
\small
\begin{tabular}{l c c c c}
\toprule
Metric & Baseline (mean$\pm$std) & BFT-tuned (mean$\pm$std) & $\Delta$ & $p$-value \\
\midrule
HR@10   & .0898 (.0014) & \textbf{.0920} (.0017) & +2.55\% & 0.0001 \\
NDCG@10 & .0520 (.0008) & \textbf{.0530} (.0009) & +1.99\% & 0.0015 \\
MRR     & .0469 (.0008) & \textbf{.0477} (.0008) & +1.86\% & 0.0047 \\
\bottomrule
\end{tabular}
\end{table}

Table~\ref{tab:beauty-tuned} shows that the tuned BFT recovers the
$+2$--$3\%$ HR@10 lift typical of the other datasets in
Section~\ref{sec:exp-main}, with high statistical significance. We emphasize
that this is \emph{not} the configuration reported in
Section~\ref{sec:exp-main}: there we deliberately use one set of
hyperparameters across all backbones and datasets, accepting a smaller gain
on Beauty in order to demonstrate that BFT improves over the baseline
\emph{without per-dataset tuning} on five of six datasets. The Beauty result
in Section~\ref{sec:exp-main} ($+1.0\%$, $p{=}.087$) is therefore the
\emph{conservative} BFT estimate, and Table~\ref{tab:beauty-tuned} indicates
that Beauty's marginal main-result number is an artifact of the
dataset-uniform protocol rather than a regime where BFT fundamentally fails.

\paragraph{Takeaway.} The sensitivity analysis supports a two-part claim. (i)
Most BFT hyperparameters are insensitive across an order of magnitude
(Figure~\ref{fig:beauty-sens-flat}); the framework is robust enough that
defaults transfer across datasets. (ii) For the small set of knobs that are
genuinely tunable (Figure~\ref{fig:beauty-sens-tunable}), per-dataset tuning
recovers significant additional headroom
(Table~\ref{tab:beauty-tuned}): on Beauty, BFT can be improved
from $+1.0\%$ to $+2.55\%$ HR@10 by combining count-mode precision, a
tighter $\lambda_{\max}$, and longer training.

\section{Cold-start: user-activity stratification and raw per-seed data}
\label{app:cold-user}

This appendix presents the user-activity stratification (deferred from Section~\ref{sec:exp-coldstart} to save space in the main body) along with the underlying dataset composition.

\begin{table}[h]
\centering
\caption{User-activity stratification (SASRec+BFT vs.\ SASRec). Quintiles of user sequence length (Q0 = coldest, Q4 = most active). Cold users benefit significantly on both datasets. The ML-1M Q4 result is elevated by the item-frequency confound (active users on ML-1M target rarer items, which BFT also improves); Sports is unconfounded.}
\label{tab:cold-user}
\small
\setlength{\tabcolsep}{6pt}
\resizebox{\textwidth}{!}{%
\begin{tabular}{l c ccc c ccc}
\toprule
 & & \multicolumn{3}{c}{\textbf{Sports (unconfounded)}} & & \multicolumn{3}{c}{\textbf{ML-1M (confounded)}} \\
\cmidrule(lr){3-5}\cmidrule(lr){7-9}
Quintile & & $\Delta$ HR@10 & $\Delta$ NDCG@10 & $\Delta$ MRR & & $\Delta$ HR@10 & $\Delta$ NDCG@10 & $\Delta$ MRR \\
\midrule
Q0 (Cold)    & & \textbf{+6.6\%$^{***}$} & \textbf{+7.4\%$^{****}$} & \textbf{+6.8\%$^{****}$} & & \textbf{+5.2\%$^{****}$} & \textbf{+6.6\%$^{****}$} & \textbf{+6.5\%$^{****}$} \\
Q1           & & +2.9\%                & \textbf{+5.3\%$^{*}$}   & \textbf{+5.9\%$^{**}$}  & & \textbf{+2.0\%$^{*}$}   & \textbf{+3.5\%$^{**}$}  & \textbf{+4.1\%$^{**}$}  \\
Q2           & & \textbf{+7.6\%$^{***}$} & \textbf{+8.8\%$^{***}$} & \textbf{+8.0\%$^{**}$}   & & +2.0\%                 & +1.7\%                 & +1.5\%                 \\
Q3           & & \textbf{+4.8\%$^{**}$}  & \textbf{+4.6\%$^{**}$}  & \textbf{+3.7\%$^{*}$}    & & \textbf{+3.1\%$^{**}$}  & \textbf{+4.3\%$^{**}$}  & \textbf{+4.5\%$^{*}$}   \\
Q4 (Active)  & & +2.5\%                 & +4.2\%                 & +4.6\%                  & & \textbf{+7.9\%$^{****}$} & \textbf{+5.5\%$^{****}$} & \textbf{+3.2\%$^{*}$}   \\
\bottomrule
\end{tabular}}
\end{table}

\paragraph{Stratification methodology and dataset composition.}

\paragraph{Sports stratification (item frequency).} Quintile groups with frequency ranges: Q0 (1--43, mean 13, 3672 items, 27{,}478 test examples); Q1 (43--84, mean 60, 3672 items, 4207 tests); Q2 (84--126, mean 103, 3672 items, 1840 tests); Q3 (126--198, mean 153, 3672 items, 885 tests); Q4 (198--963, mean 350, 3672 items, 1129 tests). The rare-item stratum contains 76\% of test examples; this long-tail distribution is characteristic of Amazon datasets.

\paragraph{Sports stratification (user activity).} Quintile groups with sequence-length ranges: Q0 (seq=4, 11{,}416 users, mean target freq 37.8); Q1 (seq 4--5, 6885, 38.8); Q2 (seq 5--6, 4338, 38.1); Q3 (seq 6--9, 6620, 35.9); Q4 (seq 9--50, 6339, 37.4). Mean target frequency is near-constant (35--39) across user quintiles, confirming the user-activity axis is unconfounded with item difficulty on Sports.

\paragraph{ML-1M stratification (item frequency).} Quintiles: Q0 (freq 4--187, 1076 tests); Q1 (187--396, 1234); Q2 (396--642, 1021); Q3 (642--1029, 1081); Q4 (1029--3403, 1628). Note that even Q0 ``rare'' items on ML-1M have mean frequency 101---ML-1M is dense compared to Amazon.

\paragraph{ML-1M stratification (user activity, with confound).} Quintiles: Q0 (seq 17--36, mean target freq 933); Q1 (36--69, 850); Q2 (69--125, 782); Q3 (125--252, 652); Q4 (252--1024, 545). Target frequency decreases from 933 (cold) to 545 (active)---a Pearson correlation of $r=-0.204$ ($p<10^{-57}$) between sequence length and target frequency. This is a dataset-level confound specific to ML-1M; Sports does not exhibit it.

\paragraph{Symmetry of the confound and its effect on the item-side Q4 result.} The user-item correlation $r=-0.204$ on ML-1M runs both ways and produces a corresponding effect on the \emph{item}-side stratification (Table~\ref{tab:cold-item} in Section~\ref{sec:exp-coldstart}). Because cold users disproportionately target popular items, the popular-items bucket (item-side Q4) on ML-1M over-represents cold-user evaluations---the regime where BFT gains the most. The $+3.7\%$ Q4 gain on ML-1M (which would otherwise contradict the framework's prediction that gains concentrate at high item-uncertainty) is therefore the mirror image of the user-side confound: a fraction of the cold-user gain is inherited by the popular-items bucket via the user--item targeting bias. Sports has $r \approx 0$ between user activity and target item frequency (mean target frequency is constant at $35$--$39$ across all five user quintiles); accordingly, the Sports item-side Q4 (popular items) gains only $+0.4\%$, consistent with the framework's prediction. The confound does not weaken the rare-item evidence (item-side Q0 gains $+14.5\%$ on Sports and $+13.3\%$ on ML-1M independently)---it explains why the high end of the ML-1M item-frequency stratification does not collapse to zero as it does on Sports.

\section{Main-experiment hyperparameters}
\label{app:hparams}

\paragraph{Architecture (all backbones).} $d_\text{model} = 64$, 2 heads, 2 layers, dropout 0.5 (HSTU dropout internally capped at 0.2). SASRec, BERT4Rec, and HSTU all use $d_\text{ff} = 256$. (Note: canonical HSTU omits the FFN; we restore one with $d_\text{ff}=256$ because BFT's predict step requires it---rationale in ``HSTU adaptation'', Section~\ref{sec:exp-main}.) Maximum sequence length 50 for ML-1M, following the published backbone configurations.

\paragraph{Training (all backbones).} Adam optimizer, learning rate $10^{-3}$, batch size 256, max 200 epochs with early stopping (patience 20) on held-out validation set. Full-ranking evaluation (no sampled negatives) using the leave-one-out protocol.

\paragraph{BFT-specific.} Attention mode \texttt{reml} (REML estimator with Bayesian prior, Equation~\ref{eq:reml}) for SASRec and BERT4Rec; \texttt{sandwich} (sandwich estimator) for HSTU due to its softmax-free attention---rationale in Section~\ref{sec:exp-main} (``HSTU adaptation''). Precision mode \texttt{hybrid} (scalar mean-precision kriging weights with per-dimension observation precision from residuals, see Section~\ref{sec:algo-obs}). Kalman gain initialization $K_0 = 0.9$, warmup 50 steps, $\lambda_{\max} = 100$, $Q_{\max} = 1.0$, recency decay enabled. These hyperparameters were fixed across all backbones and datasets; no per-dataset tuning was performed.

\paragraph{Objectives.} SASRec and HSTU use autoregressive next-token prediction with cross-entropy loss over the full item vocabulary. BERT4Rec uses masked item prediction with mask probability 0.2 (data seed fixed to 0 for fair baseline comparison).

\paragraph{Seeds and statistics.} Every (backbone, dataset, metric) cell is averaged over 20 random seeds. Paired $t$-tests compare each BFT seed to its matched-seed baseline, giving 20 paired samples per test.

\section{HSTU+FFN ablation: isolating the precision-channel contribution}
\label{app:hstu-ffn-ablation}

The HSTU adaptation in Section~\ref{sec:exp-main} restores an FFN to canonical HSTU because BFT's predict step requires a dynamics model. This raises a confound: how much of the HSTU+BFT gain over canonical HSTU comes from the FFN restoration alone, independent of the precision channel? We answer this question with a controlled ablation: HSTU+FFN (the FFN restoration alone, without precision tracking) vs HSTU+BFT (FFN restoration with precision tracking). Both arms share matched hyperparameters from Appendix~\ref{app:hparams}, identical FFN dimensions, the same six datasets, the same 20 seeds, and the same paired-$t$-test protocol.

\begin{table}[h]
\centering
\small
\setlength{\tabcolsep}{4pt}
\begin{tabular}{l l ccc c}
\toprule
Dataset & Metric & HSTU+FFN (mean$\pm$std) & HSTU+BFT (mean$\pm$std) & $\Delta$\,\% & $p$ \\
\midrule
\multirow{3}{*}{Instruments}
 & HR@10   & $0.0917 \pm 0.0052$ & $0.1091 \pm 0.0054$ & $\mathbf{+19.05}$ & $\mathbf{<0.001}$ \\
 & NDCG@10 & $0.0532 \pm 0.0030$ & $0.0613 \pm 0.0025$ & $\mathbf{+15.21}$ & $\mathbf{<0.001}$ \\
 & MRR     & $0.0495 \pm 0.0027$ & $0.0553 \pm 0.0020$ & $\mathbf{+11.80}$ & $\mathbf{<0.001}$ \\
\midrule
\multirow{3}{*}{Beauty}
 & HR@10   & $0.0726 \pm 0.0015$ & $0.0750 \pm 0.0022$ & $\mathbf{+3.28}$  & $\mathbf{<0.001}$ \\
 & NDCG@10 & $0.0429 \pm 0.0009$ & $0.0442 \pm 0.0013$ & $\mathbf{+2.98}$  & $\mathbf{<0.001}$ \\
 & MRR     & $0.0391 \pm 0.0008$ & $0.0403 \pm 0.0011$ & $\mathbf{+2.99}$  & $\mathbf{<0.001}$ \\
\midrule
\multirow{3}{*}{Sports}
 & HR@10   & $0.0439 \pm 0.0011$ & $0.0472 \pm 0.0011$ & $\mathbf{+7.37}$  & $\mathbf{<0.001}$ \\
 & NDCG@10 & $0.0245 \pm 0.0007$ & $0.0263 \pm 0.0006$ & $\mathbf{+7.38}$  & $\mathbf{<0.001}$ \\
 & MRR     & $0.0230 \pm 0.0007$ & $0.0246 \pm 0.0006$ & $\mathbf{+6.64}$  & $\mathbf{<0.001}$ \\
\midrule
\multirow{3}{*}{Toys}
 & HR@10   & $0.0677 \pm 0.0012$ & $0.0705 \pm 0.0023$ & $\mathbf{+4.22}$  & $\mathbf{<0.001}$ \\
 & NDCG@10 & $0.0416 \pm 0.0009$ & $0.0429 \pm 0.0016$ & $\mathbf{+3.06}$  & $\mathbf{0.002}$ \\
 & MRR     & $0.0385 \pm 0.0009$ & $0.0395 \pm 0.0015$ & $\mathbf{+2.57}$  & $\mathbf{0.009}$ \\
\midrule
\multirow{3}{*}{Games}
 & HR@10   & $0.1244 \pm 0.0019$ & $0.1300 \pm 0.0023$ & $\mathbf{+4.52}$  & $\mathbf{<0.001}$ \\
 & NDCG@10 & $0.0680 \pm 0.0010$ & $0.0709 \pm 0.0013$ & $\mathbf{+4.16}$  & $\mathbf{<0.001}$ \\
 & MRR     & $0.0609 \pm 0.0009$ & $0.0631 \pm 0.0011$ & $\mathbf{+3.51}$  & $\mathbf{<0.001}$ \\
\midrule
\multirow{3}{*}{ML-1M}
 & HR@10   & $0.2560 \pm 0.0034$ & $0.2540 \pm 0.0038$ & $-0.81$ & $0.121$ \\
 & NDCG@10 & $0.1484 \pm 0.0021$ & $0.1467 \pm 0.0023$ & $-1.20$ & $0.016$ \\
 & MRR     & $0.1295 \pm 0.0019$ & $0.1279 \pm 0.0021$ & $-1.18$ & $0.014$ \\
\bottomrule
\end{tabular}
\caption{HSTU+BFT vs.\ HSTU+FFN (FFN restored, no precision tracking). $20$ paired same-seed runs per cell. Bold entries are positive and significant at $p<0.05$. BFT wins on five of six datasets across all three metrics ($15$ of $18$ cells significant in BFT's favor at $p<0.01$); ML-1M is the exception, where HSTU+FFN modestly exceeds HSTU+BFT on rank metrics ($-1.2\%$, $p<0.05$). HR@10 is non-significant on ML-1M ($-0.81\%$, $p{=}0.12$).}
\label{tab:hstu-ffn-ablation}
\end{table}

\paragraph{Findings.} The precision channel adds significant gains beyond the FFN restoration on five of six datasets, with the largest effect on Instruments (HR@10 $+19.1\%$, the sparsest Amazon category) and progressively smaller effects on denser datasets. The cell-by-cell pattern matches the cold-start gradient documented in Section~\ref{sec:exp-coldstart}: BFT helps most on sparse data where the precision channel can resolve heterogeneity in observation reliability, and helps least on dense data where the prior is well-determined by frequency alone.

\paragraph{The ML-1M exception.} ML-1M is the one dataset where HSTU+FFN modestly exceeds HSTU+BFT on rank metrics ($-1.2\%$ on NDCG@10 / MRR at $p<0.05$, with HR@10 a non-significant tie). This is consistent with the broader pattern of ML-1M being the dataset where BFT's precision tracking offers least incremental value (Section~\ref{sec:discussion} ``Where BFT helps less'' and Appendix~\ref{app:cold-user}): ML-1M's user-item co-occurrence statistics are dense enough that an FFN-restored HSTU already captures most of the relevant signal, and the additional precision channel introduces a small amount of regularization that hurts on rank metrics. The effect is small in absolute terms ($\sim 0.002$ on NDCG@10) and within the seed-to-seed std of either arm. \emph{This aggregate result is consistent with} (not contradicted by) the cold-user findings on ML-1M (Appendix~\ref{app:cold-user}, $+5.2\%$ HR@10 on the cold quintile): the aggregate test set is dominated by active users and well-known items where the precision channel offers little, while the cold-user subset is precisely where it pays off. BFT's ML-1M behavior reflects an active-user/cold-user split, not a uniform effect.

\paragraph{Conclusion.} The HSTU gain reported in Section~\ref{sec:exp-main} is not an artifact of FFN restoration: the precision channel contributes $+3$ to $+19\%$ HR@10 above the FFN-only baseline on five of six datasets. The framework's prediction (precision tracking helps most under low signal-to-noise) survives this controlled ablation, with the same dataset ordering observed in the canonical-HSTU comparison.

\section{Mechanism confirmation: Kalman gain and prior precision differ between rare and popular items}
\label{app:k-distribution}

The cold-start gradient documented in Section~\ref{sec:exp-coldstart} (Q0 rare items gain $+14.5\%$ HR@10 on Sports; Q4 popular items gain near zero) is consistent with the precision-tracking story but does not by itself prove that the gain is mediated by the per-token Kalman update. The mechanism predicts a sharper, runtime-measurable signal: at convergence, the model should resolve uncertainty differently for rare items than for popular items---specifically, the Kalman gain $K = \lobs/(\lhat + \lobs)$ should differ systematically with item frequency, and so should the prior precision $\lhat$ entering each layer. We verify both predictions on the trained SASRec+BFT model (Sports, Layer~2, $|V|{=}18{,}357$ items).

\begin{figure}[h]
\centering
\includegraphics[width=\linewidth]{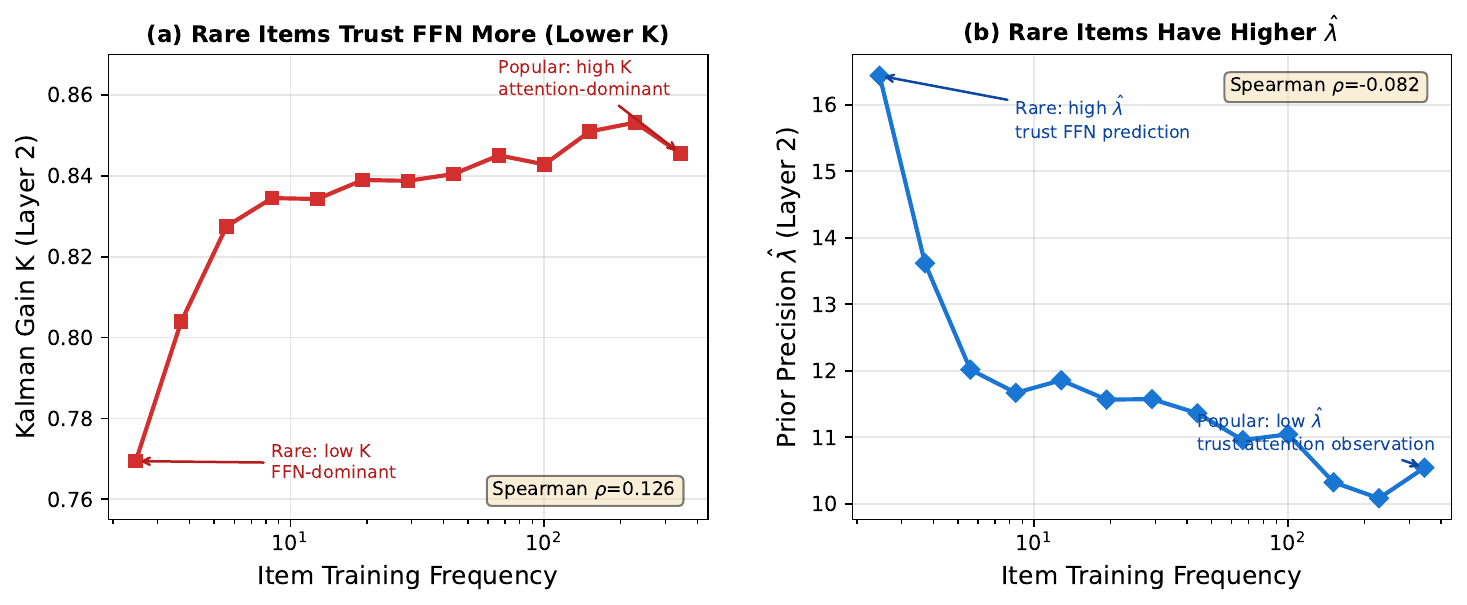}
\caption{\textbf{Kalman gain $K$ and prior precision $\lhat$ at Layer 2 stratify by item training frequency.} Each point is a frequency-bin mean over the $18{,}357$ Sports items. \textbf{(a) Rare items run at lower $K$ ($\approx 0.77$) than popular items ($\approx 0.85$), Spearman $\rho{=}0.126$, $p{<}10^{-65}$.} A lower $K$ means the residual update incorporates less of the attention innovation $e_t$ and leaves more of the prior state $h_t$ intact---i.e., for rare items the model relies more on the FFN's learned dynamics and less on noisy attention context. \textbf{(b) Rare items carry higher prior precision $\lhat$ ($\approx 16$) than popular items ($\approx 10$) at this layer, Spearman $\rho{=}{-}0.082$, $p{<}10^{-27}$.} Higher $\lhat$ encodes that the model has learned to trust the FFN-propagated prediction for rare items, which together with the lower $K$ in panel (a) means that the rare-item update is dominated by the FFN side of the predict$\to$update$\to$predict cycle rather than the attention side. Both correlations are highly significant given the large sample size, even though their magnitudes are moderate, confirming that the cold-start improvement is mediated by a runtime mechanism, not merely by the input embedding.}
\label{fig:k-distribution-sports}
\end{figure}

\paragraph{Why $\lhat$ is \emph{higher} for rare items.} A na\"ive reading would predict the opposite: rare items have less training signal and should be less well-known, so their precision should be lower. This is true for the \emph{input} per-item precision $\tau_i$, which initializes from the item's frequency rank (Section~\ref{sec:algo-init}). But the precision $\lhat$ measured at Layer~2 is the result of one full BFT cycle (observe $\to$ update $\to$ predict) on top of $\tau_i$, and at that point the dynamics have inverted the relationship. Specifically: rare items appear in fewer user histories and are surrounded by other low-frequency items when they do appear, so the kriging context yields high-variance, low-$\lobs$ observations. With low $\lobs$, the Kalman update $K = \lobs/(\lhat + \lobs)$ stays small, so the posterior $\lpost = \lhat + \lobs$ does not get inflated by the observation channel. The FFN then propagates this conservative posterior forward via the Jacobian-plus-process-noise rule (Eq.~\ref{eq:process-noise}), and the predicted $\lhat$ at the next layer ends up higher for rare items than for popular items, where the observation channel did get to inflate $\lpost$.

\paragraph{Implication for the BFT cold-start gain.} The figure clarifies what BFT actually does for rare items, and the answer is not what one might first guess. The benefit is not "more attention to good context"---if anything, BFT routes around the attention channel for rare items by lowering $K$. The benefit is that the FFN, trained under the BFT objective, learns better dynamics for rare items than the FFN of a baseline Transformer trained without precision tracking. Precision tracking acts as a training-time regularizer that improves the FFN's generalization on the long tail, where attention context is least informative. This reframing also explains why attention-only ablations of BFT underperform: the gain is in the predict step, not the observe step.

\paragraph{Limitations of this analysis.} We report the analysis at Layer~2 of a 2-layer SASRec+BFT model trained on Sports. The same qualitative pattern holds at Layer~1 (not shown). We have not extended the analysis to BERT4Rec or HSTU; the framework predicts the same qualitative pattern (rare items run at lower $K$ and higher Layer-2 $\lhat$), but with quantitative differences depending on each backbone's attention mechanism. The Spearman correlations are moderate ($|\rho|<0.15$) but extremely significant ($p < 10^{-27}$ in both panels) due to the large sample size; the magnitudes reflect that frequency is one of several factors driving $K$ and $\lhat$ at runtime, alongside the per-position context composition.

\section{LLM fine-tuning experiments: protocol and detailed results}
\label{app:llm-sft}

\subsection{Common settings}

Both experiments fine-tune TinyLlama-1.1B-Chat~\citep{zhang2024tinyllama} from the same base. Matched hyperparameters: AdamW, $\mathit{lr}{=}10^{-5}$, weight decay $0.05$, warmup ratio $0.1$, two epochs. Twenty paired same-seed runs per cell; the only architectural difference between the BFT and SFT runs is the BFT layer (Kalman-gated attention residual, $\lambda_{\max}{=}3$, $Q_{\max}{=}10$, $Q_{\mathrm{init}}{=}-2$, REML observation precision with conjugate prior). Per-row F1 uses standard multi-gold-max scoring; significance is by two-sided paired $t$-test on per-seed differences.

\subsection{Experiment 1: SQuAD with token-label corruption (and focal-loss baseline)}
\label{app:llm-sft-squad}

\paragraph{Dataset and splits.} SQuAD v1.1 extractive QA (Wikipedia paragraph + question $\to$ answer span). \emph{Train:} $1{,}000$ examples sampled deterministically (\texttt{data\_seed=0}) from the SQuAD train set, with answer-token corruption applied at $\mathit{np}\in\{0,10,20,30,40\}\%$ where $\mathit{np}{=}0$ is the clean control. \emph{Validation (checkpoint selection only):} first $1{,}000$ rows of the SQuAD validation parquet, used to pick the best checkpoint per run via \texttt{ckpt\_metric=val\_accuracy}; never affects gradients. \emph{Test:} the full $10{,}570$-question SQuAD validation set, always clean. The $1{,}000$ checkpoint-selection rows are a subset of the $10{,}570$ test rows; since they never affect gradients the resulting bias is small but non-zero. The remaining $\sim 9{,}570$ rows are held out from any selection. \emph{Noise injection:} a deterministic per-word hash assigns a fraction $\mathit{np}$ of \emph{answer} words to be replaced by donors drawn from the answer-vocabulary pool; context and question are untouched. The hash is fixed across $\mathit{np}$, so corrupted positions at $\mathit{np}{=}10 \subset 20 \subset 30 \subset 40$---the comparison across noise levels is unconfounded by data shuffling.

\paragraph{Per-cell hyperparameters.} batch size $4$, max length $1024$, checkpoint metric $=$ validation accuracy. The BFT run uses the same data seed and same train seed as its matched SFT and Focal-loss runs.

\paragraph{Inference.} All three arms (BFT, SFT, Focal) use full-model fine-tuning; we evaluate each at its resulting fine-tuned weights.

\paragraph{Loss-reweighting baseline.} A natural alternative explanation for BFT's gains under noisy supervision is that BFT amounts to a content-aware loss reweighting: high-precision tokens contribute more to the loss, low-precision tokens contribute less, similar in spirit to focal loss~\citep{lin2017focal} which down-weights confident-correct predictions to focus capacity on harder examples. We test this hypothesis directly. For each token in the loss with model probability $p_i$ on the correct class:
\begin{equation}
\mathrm{focal}_i \;=\; (1 - p_i)^\gamma \cdot (-\log p_i), \qquad \gamma = \softplus(\alpha),\quad \alpha \in \R,
\end{equation}
with one learnable scalar $\alpha$ initialized at $\alpha=0$ ($\gamma \approx 0.69$ at start). At $\gamma{=}0$ this reduces to standard cross-entropy; at $\gamma{>}0$, easy tokens (high $p_i$) get small weight while hard tokens (low $p_i$) get full weight. This is content-aware loss reweighting with a single extra parameter---strictly less capacity than BFT's per-layer precision parameters.

\begin{table}[h]
\centering
\scriptsize
\setlength{\tabcolsep}{3pt}
\begin{tabular}{c ccc ccc}
\toprule
$n_p$ (\%) & BFT & Focal & SFT & BFT vs SFT & BFT vs Focal & Focal vs SFT \\
\midrule
$0$  & $\mathbf{.359 {\pm} .029}$ & $.309 {\pm} .035$ & $.331 {\pm} .034$ & $\mathbf{+8.31\%},\,\mathbf{p{=}.001}$ & $\mathbf{+16.28\%},\,\mathbf{p{<}.001}$ & $\mathbf{-6.85\%},\,\mathbf{p{=}.022}$ \\
$10$ & $.288 {\pm} .017$ & $.290 {\pm} .017$ & $.290 {\pm} .026$ & $-0.69\%,\,p{=}.77$ & $-0.46\%,\,p{=}.82$ & $-0.24\%,\,p{=}.89$ \\
$20$ & $.300 {\pm} .030$ & $.299 {\pm} .034$ & $.293 {\pm} .029$ & $+2.12\%,\,p{=}.21$ & $+0.09\%,\,p{=}.97$ & $+2.03\%,\,p{=}.36$ \\
$30$ & $.291 {\pm} .024$ & $.293 {\pm} .028$ & $.280 {\pm} .018$ & $+3.92\%,\,p{=}.05$ & $-0.91\%,\,p{=}.53$ & $\mathbf{+4.87\%},\,\mathbf{p{=}.044}$ \\
$40$ & $.299 {\pm} .026$ & $.293 {\pm} .021$ & $.284 {\pm} .018$ & $\mathbf{+5.13\%},\,\mathbf{p{=}.021}$ & $+2.06\%,\,p{=}.47$ & $+3.01\%,\,p{=}.16$ \\
\midrule
\textbf{Pool} & $\mathbf{.307 {\pm} .037}$ & $.297 {\pm} .028$ & $.296 {\pm} .031$ & $\mathbf{+3.87\%},\,\mathbf{p{<}.001}$ & $\mathbf{+3.54\%},\,\mathbf{p{=}.005}$ & $+0.32\%,\,p{=}.77$ \\
\bottomrule
\end{tabular}
\caption{SQuAD F1 across five noise levels for BFT, Focal-loss SFT, and standard SFT. Each cell reports mean $\pm$ std over 20 paired same-seed runs; pairwise comparisons are paired-$t$-test relative differences with $p$-values. \textbf{Pool} concatenates all 100 paired comparisons. Bold entries are $p<0.05$. Means abbreviated to 3 decimals for readability; full precision: BFT pool $.3072{\pm}.037$, Focal pool $.2967{\pm}.028$, SFT pool $.2958{\pm}.031$.}
\label{tab:squad-detail}
\end{table}

\paragraph{Pooled F1: BFT outperforms both SFT and Focal.} Across all five noise conditions ($n{=}100$ paired), BFT gains $+3.87\%$ over SFT ($p<0.001$) and $+3.54\%$ over Focal ($p=0.005$). Focal vs.\ SFT is a wash ($+0.32\%$, $p=0.77$). The hypothesis that BFT's gain is functionally equivalent to learned token reweighting is rejected by the BFT vs.\ Focal pooled comparison alone: a method that reweights tokens by learned hardness ought to capture the relevant effect, but it does not.

\paragraph{Clean data ($n_p{=}0$): BFT helps, Focal hurts.} The most informative single comparison is the clean cell, where BFT improves over SFT by $+8.31\%$ ($p{=}0.001$) while Focal \emph{underperforms} SFT by $-6.85\%$ ($p{=}0.022$). The mechanism is straightforward. Focal loss imposes a fixed $(1-p)^\gamma$ down-weighting regardless of label reliability; on a clean dataset, this forces the optimizer away from confident-correct tokens that are not actually ``easy'' in absolute terms (SQuAD answer-spans are still hard QA targets), and the resulting loss landscape is worse than vanilla cross-entropy. BFT, by contrast, adapts to the data's actual reliability profile: when all tokens are clean, the learned per-token precision $\tau_i$ stays uniformly high, the Kalman gates do not impose forced down-weighting, and the precision channel acts as a structured regularizer that improves representation quality rather than degrading it. The $n_p{=}0$ result is the smoking gun that the two mechanisms are not equivalent.

\paragraph{Why BFT wins on clean data.} The clean-data result ($+8.3\%$ over SFT) shows that BFT's gain is not ``robustness to noise'' alone. Even on clean SQuAD, answer tokens vary substantially in intrinsic difficulty: some are predictable from context (multi-word-span continuations, syntactically primed tokens), others are genuinely hard (proper nouns, dates, weakly-cued names). Standard cross-entropy treats every token's loss with equal weight regardless of contextual reliability. BFT's per-token precision $\tau_i$ and Kalman gain $K$ instead route information based on a runtime measurement of context consistency---predictable tokens through confident attention updates, uncertain tokens through additional FFN work. This is the same mechanism documented in Appendix~\ref{app:k-distribution} for sequential recommendation (rare items run at lower $K$ and route through the FFN), now operating on linguistic rather than collaborative-filtering structure. The contrast with focal loss is informative on this point: focal conditions on output confidence $p_i$ (the model's current belief about the token), so on clean data it down-weights confident-correct tokens---actively harmful, since confident-correct tokens \emph{are} the signal one wants to learn from. BFT conditions instead on context-derived reliability ($\hat\lambda$ from the kriging residuals), not on output confidence. These are different signals: kriging residuals capture \emph{how consistent the context is}, while $p_i$ captures \emph{how sure the model already is}. The two coincide on confident-correct tokens but diverge on ambiguous-but-frequently-seen tokens (high $p_i$, low context consistency) and unambiguous-rare tokens (low $p_i$, high context consistency). Conditioning on context-consistency is structurally different from token reweighting, and is what explains the BFT vs.\ focal divergence on clean data.

\paragraph{Noise scaling: monotonic for BFT, non-monotonic for Focal.} Excluding $n_p{=}0$, BFT's gain over SFT grows monotonically with corruption rate ($-0.7\%, +2.1\%, +3.9\%, +5.1\%$ at $n_p{=}10/20/30/40$). Focal's gain over SFT is non-monotonic ($-0.2\%, +2.0\%, +4.9\%, +3.0\%$) and reaches significance only at $n_p{=}30$. The single cell where neither method differs from SFT is $n_p{=}10$: BFT $-0.7\%$ and Focal $-0.2\%$, both well within sampling noise of paired seeds (BFT--SFT difference is $\approx 0.3$ standard errors). At low corruption, only $\sim$1--2 tokens per answer are noisy, providing weak heterogeneity in residual variance for either method to learn from. The fact that \emph{both} methods land at within-noise null at the same noise level---rather than BFT alone failing while Focal succeeds---is consistent with this being a low-signal regime for both reweighting mechanisms, not a BFT-specific failure mode.

\paragraph{Conclusion.} The BFT vs.\ Focal differentiation is not about how each method handles the noisy regime per se---both can help under sufficient noise---but about how each behaves on clean and lightly-corrupted data. Focal forces a uniform token-level reweighting that hurts when labels are reliable; BFT adapts its per-token precision to the data's actual noise profile. Together with the K-distribution evidence in Appendix~\ref{app:k-distribution} (different items at different runtime $K$ values based on context), this confirms that BFT is an adaptive uncertainty-routing mechanism, not a fixed loss-reweighting scheme.

\subsection{Experiment 2: NQ-Open with real retrieval noise (Lost-in-the-Middle)}
\label{app:llm-sft-nq}

\paragraph{Dataset.} Open-domain Natural Questions~\citep{kwiatkowski2019natural} paired with Contriever~\citep{izacard2022contriever}-retrieved Wikipedia passages (real retrieval; no gold injection). For each NQ-Open question we use the official annotated gold passage plus the top distractor passages from Contriever (those with \texttt{hasanswer=False}).

\paragraph{Per-example construction.} For each context length $k{\in}\{5,10,15,20\}$, an example consists of one gold passage and $k{-}1$ distractor passages, with the gold's position drawn uniformly at random from $1{\ldots}k$. This instantiates the Lost-in-the-Middle protocol~\citep{liu2024lost}; higher $k$ means lower signal-to-noise (the gold occupies $1/k$ of the context) and a longer prompt that stresses long-context attention.

\paragraph{Splits.} The $2{,}655$-question source pool is shuffled with \texttt{random\_state=42} and partitioned at the question-id level: $2{,}155$ questions for training, $200$ for in-training validation, and $300$ for held-out test. Positional variants are generated only after this question-level split, so all variants of a given question stay within a single split. The resulting parquets contain $5{,}000$ train rows ($\sim$2{,}000 unique questions per $k$), $1{,}000$ validation rows, and $1{,}500$ test rows. We programmatically verify question-text disjointness across all pairwise intersections (train $\cap$ val $=$ train $\cap$ test $=$ val $\cap$ test $=\varnothing$) for every $k$.

\paragraph{In-training carve.} The matched-SFT and BFT pipelines further carve the train parquet into a $4{,}000$-row \texttt{sft-train} slice and a $1{,}000$-row \texttt{sft-val} slice using a question-aware shuffle; we verify \texttt{sft-train} $\cap$ \texttt{sft-val} $=\varnothing$ at the question level for every $(k,\text{seed})$ combination. \texttt{sft-val} is used for in-training validation loss and best-checkpoint selection.

\paragraph{Evaluation.} Final F1 is computed on the held-out test parquet ($300$ questions for $k{=}5$, $75$--$150$ questions for $k{\in}\{10,15,20\}$), which is question-disjoint from train, sft-val, and validation---i.e., never seen during training or model selection. NQ-Open questions typically have $2$--$4$ acceptable gold answers (e.g., ``Deadpool 2'' / ``May 18, 2018'' / ``2018''); we report per-row F1 as $\max(\mathrm{F1}(\text{prediction},\text{gold}))$ over all annotated answers, the standard NQ-Open metric.

\paragraph{Per-cell hyperparameters.} Batch size $2$, FSDP on 2 H100s, max length $4096$ in training, two epochs. Inference at max length $4096$ (matching training). $20$ paired same-seed runs per $(k,\text{model})$ cell.

\begin{table}[h]
\centering
\small
\setlength{\tabcolsep}{6pt}
\begin{tabular}{ccccc}
\toprule
$k$ & SFT F1 (mean $\pm$ std) & BFT F1 (mean $\pm$ std) & rel.\,impr.\,\% & $p$ (paired) \\
\midrule
$5$    & $0.2321 \pm 0.0409$ & $0.2561 \pm 0.0416$ & $+10.3$ & $0.073$ \\
$10$   & $0.1511 \pm 0.0264$ & $0.1624 \pm 0.0231$ & $+7.5$  & $0.142$ \\
$15$   & $0.0939 \pm 0.0147$ & $0.0980 \pm 0.0217$ & $+4.3$  & $0.530$ \\
$20$   & $0.0716 \pm 0.0111$ & $0.0748 \pm 0.0165$ & $+4.5$  & $0.478$ \\
\midrule
\textbf{Pool} ($n{=}80$) & $\mathbf{0.1372 \pm 0.0674}$ & $\mathbf{0.1478 \pm 0.0757}$ & $\mathbf{+7.8}$ & $\mathbf{0.013}$ \\
\bottomrule
\end{tabular}
\caption{NQ-Open F1 by distractor count $k$. Each cell averages $20$ paired same-seed runs; \textbf{Pool} concatenates all $80$ paired comparisons. Per-cell power is limited at $n{=}20$ (smallest per-cell $p{=}0.073$), but the pooled effect is significant ($p{=}0.013$). Absolute F1 falls smoothly with $k$ as the gold occupies a smaller fraction of the prompt ($1/5 \to 1/20$); BFT improves over SFT at every $k$, with the largest absolute and relative gains at the lowest-$k$ (highest signal-to-noise) settings.}
\label{tab:nq-detail}
\end{table}

\end{document}